\documentclass[10pt]{article}

\usepackage[BFP,NWS,glossary]{stylefile_FR_2025_feb_28}
\usepackage{bbm}

\newcommand{\ExpP}[2]{\mathbb{E}_{#1}\left[#2\right]}
\newcommand{\ExpC}[3]{\mathbb{E}_{#1}\left[#2\,|\,#3\right]}

\newcommand{\Seq}[1]{\underline{#1}}
\newcommand{\seq}[2]{\underline{#1}_{#2}}

\newcommand{\US}{\underline{\Uc}}
\newcommand{\XS}{\underline{\Xc}}
\newcommand{\ZS}{\underline{\Zc}}

\title{State-space systems as dynamic generative models}
\author{Juan-Pablo Ortega\footnote{Division of Mathematical Sciences, School of Physical and Mathematical Sciences, Nanyang Technological University, Singapore} \qquad Florian Rossmannek$^*$}
\Headings{State-space systems as dynamic generative models}{Ortega and Rossmannek}

\begin{document}

\maketitle

\begin{abstract}
A probabilistic framework to study the dependence structure induced by deterministic discrete-time state-space systems between input and output processes is introduced.
General sufficient conditions are formulated under which output processes exist and are unique once an input process has been fixed, a property that in the deterministic state-space literature is known as the echo state property.
When those conditions are satisfied, the given state-space system becomes a generative model for probabilistic dependences between two sequence spaces.
Moreover, those conditions guarantee that the output depends continuously on the input when using the Wasserstein metric.
The output processes whose existence is proved are shown to be causal in a specific sense and to generalize those studied in purely deterministic situations.
The results in this paper constitute a significant stochastic generalization of sufficient conditions for the deterministic echo state property to hold, in the sense that the stochastic echo state property can be satisfied under contractivity conditions that are strictly weaker than those in deterministic situations.
This means that state-space systems can induce a purely probabilistic dependence structure between input and output sequence spaces even when there is no functional relation between those two spaces.
\end{abstract}

\keys{state-space system, reservoir computing, echo state property, fading memory property, Wasserstein distance, generative model.}


\section{Introduction}

Given two phenomena modelled by two related random variables $X$ and $Y$ that take values, in general, in different spaces, much effort in statistics and machine learning is dedicated to estimating or uncovering their joint probability distribution $P(X, Y)$ out of data.
That knowledge allows us to make predictions of the {\bfi explained} (dependent) {\bfi variables} $Y$ out of knowledge about the {\bfi explanatory variables} (covariates) $X$.
Models that allow us to estimate $P(X, Y)$ are called {\bfi generative} since once that distribution is known, Bayes' law makes possible the generation of samples of the dependent variable $Y$ out of realizations of the covariates $X$.
In practice, $X$ may be just a prescribed noise exclusively used to generate the endogenous variable $Y$, exogenous explanatory factors with a known or estimated out-of-data distribution, or a mixture of the two.
This problem acquires a different breadth in dynamic setups, that is when $X$ and $Y$ are stochastic processes or time series models since, in both cases, these two random variables map into infinite dimensional spaces (paths or sequences, respectively).
This is the context in which we shall work in this paper;
more specifically, we place ourselves in a discrete-time setup where covariates and dependent variables are stochastic time series models.
When the covariates are independent or white noise, we recover the setup customary in parametric time series analysis \cite{BrockwellDavis2006, FrancqZakoian2019};
if they incorporate a mixture of noise with explanatory variables with a known law, we recover exogenous time series models, a prominent example of which is the NARMAX family \cite{Billings2013}.

A useful tool to establish functional relations between sequence spaces in deterministic setups are {\bfi state-space systems} \cite{JiangLiLiWang2023JML, Kalman1960, Sontag1990}.
We briefly recall how that is done.
In the sequel, we adopt the notation for input and output spaces customary in that field.
Let $\Xc$ and $\Uc$ be topological spaces, and let $f \colon \Xc \times \Uc \rightarrow \Xc$ 
\nomenclature[57]{$f \colon \Xc \times \Uc \rightarrow \Xc$}{State equation}%
be a continuous {\bfi state map}.
The set $\Xc$ 
\nomenclature[53]{$(\Xc,d_{\Xc})$}{State space}%
is called the {\bfi state space} and $\Uc$ is the {\bfi input space}.
\nomenclature[55]{$\Uc$}{Input space}%
Now, denote by $\Z_-$ 
\nomenclature[02]{$\Z_-$}{Set of strictly negative integers}%
the set of strictly negative integers and extend the state map $f$ to a map between sequence spaces by defining $F \colon \Xc^{\Z_-} \times \Uc^{\Z_-} \rightarrow \Xc^{\Z_-}$ 
\nomenclature[60]{$F \colon \Xc^{\Z_-} \times \Uc^{\Z_-} \rightarrow \Xc^{\Z_-}$}{Extension of state equation to sequence space}%
as $F(\Seq{x},\Seq{u})_t = f(\seq{x}{t-1},\seq{u}{t})$.
Throughout the paper, all topological spaces are endowed with their Borel $\sigma$-algebras, and the Cartesian product of topological spaces is equipped with the product topology unless indicated otherwise.
Under these hypotheses, $F$ inherits continuity from $f$.

We recall that a {\bfi solution} of the state-space system $f$ is an element $(\Seq{x},\Seq{u}) \in \Xc^{\Z_-} \times \Uc^{\Z_-}$ that satisfies $\Seq{x} = F(\Seq{x},\Seq{u})$.
We also call $\Seq{x}$ a {\bfi solution} for the {\bfi input} $\Seq{u}$.
We say that the state-space system has the {\bfi echo state property} \cite{Jaeger2010} if for all $\Seq{u} \in \Uc^{\Z_-}$ there exists a unique solution $\Seq{x} \in \Xc^{\Z_-}$ for $\Seq{u}$.
In this case, a functional relation between the sequence spaces on the input and the state spaces can be established using the associated {\bfi filter} $U_f \colon \Uc^{\Z_-} \rightarrow \Xc^{\Z_-}$ 
\nomenclature[59]{$U_f \colon \Uc^{\Z_-} \rightarrow \Xc^{\Z_-}$}{Filter associated to the state map $f$ (if existent)}%
that is uniquely determined by the property $U_f(\Seq{u}) = F(U_f(\Seq{u}),\Seq{u})$.
Much work has been done in the literature to determine conditions under which the echo state property holds \cite{Jaeger2010, RC7} or those that make the corresponding filter continuous or (Fr{\'e}chet) differentiable with respect to various Banach space structures on the sequence spaces \cite{RC9}.
Other papers have studied in detail the various dynamical implications of this property \cite{Manjunath2020ProcA, Manjunath2022Nonlin}.

The main theme of this paper is to study the dependence structure created by state-space systems when we replace the sequences in the previous statements by {\bfi stochastic inputs} and {\bfi outputs}.
The simplest situation is when we can assume that the filter $U_f $ exists and it is measurable, and then instead of an input sequence $\Seq{u} $ we consider a $\Uc^{\Z_-}$-valued random variable $\Seq{U}$.
In that situation, we call the process $U_f(\Seq{U})$ the {\bfi stochastic output} of the system.
On the level of measures, if $\Xi$ 
\nomenclature[92]{$\Xi$}{Law of the input process}%
is the law of the input process, then the push-forward measure $(U_f)_*\Xi$ is the law of the output process.

We shall see later that the echo state property in the stochastic situation is much richer than in the deterministic one since we can typically make sense of stochastic outputs and, more generally, dependence structures between input and output sequences, even if deterministic solutions (filters) do not exist.
Formulated differently, state-space systems can create dependences between input and output processes that are purely probabilistic and are not backed by a deterministic functional relation between the corresponding sequence spaces.
Indeed, a state-space system can admit several deterministic solutions for the same deterministic input or no solution at all but still admit a (unique) stochastic output.
This observation has already been made in \cite{ManjunathOrtega2023}, where results along these lines were first formulated using a non-autonomous version of the {\bfi Frobenius-Perron operator} (dual of the {\bfi Koopman operator} that appears in dynamical systems theory) called the {\bfi Foias operator} \cite{LasotaMackey1994}.
The use of this object makes the results in that paper valid only for a very limited class of stochastic inputs (those that do not induce probabilistic dependences with the output), a restriction that will be circumvented in this paper using a more general functional analytic approach applicable to more general input spaces.
This, in passing, will shed light on the causality properties of stochastic outputs generated by state-space systems.

Our main motivation for studying the echo state property in the stochastic setup is the possibility of learning dynamic dependences between stochastic explanatory and explained variables as a generalization of learning techniques already used in deterministic situations.
Indeed, the echo state property was historically introduced in the field of {\bfi reservoir computing}.
This technique \cite{JaegerHaas2004, Jaeger2010} uses randomly generated state-space systems to learn the attractor of deterministic dynamical systems \cite{ArcomanoEtal2022, LuHuntOtt2018Chaos, PathakEtal2017Chaos, PathakEtal2018PRL, WiknerEtal2021Chaos} as well as dynamic input/output dependences \cite{JiangLiLiWang2023JML, RC10, RC12, RC18, RC21} between deterministic variables.
A natural generalization of the latter results would consist of using the same systems to construct dynamic non-linear regressions in which the covariates and the dependent variables are not sequences anymore but stochastic time series, a framework that is prevalent in many modelling situations \cite{BrockwellDavis2006, FrancqZakoian2019}.
The first question that needs to be answered in that case is under what conditions a state-space model produces dependent stochastic variables when fed with stochastic covariates.
Having a good understanding of this relationship is central to the work in this paper. 

In the deterministic setup, the echo state property is traditionally ensured \cite{Jaeger2010} using a contractivity hypothesis for the state map $f$ on the state variables.
Our main result (\cref{Banach_fp_dep_input}) replaces this condition with a more general stochastic version that involves not only $f$ but also the law of the stochastic input.
To explain this statement in some detail, we introduce the map $\Fc \colon \Xc^{\Z_-} \times \Uc^{\Z_-} \rightarrow \Xc^{\Z_-} \times \Uc^{\Z_-}$ 
\nomenclature[61]{$\Fc \colon \XS \times \US \rightarrow \XS \times \US$}{Sequence space extension of $f$ to match domain and codomain}%
given by $\Fc(\Seq{x},\Seq{u}) = (F(\Seq{x},\Seq{u}),\Seq{u})$ that has been built in such a way that the solutions to the state-space system are characterized as fixed points of $\Fc$.
The associated filter $U_f$, if it exists, satisfies $\Fc \circ (U_f \times \mathrm{id}) = U_f \times \mathrm{id}$.
The law $(U_f)_*\Xi$ of the output process is characterized as being the first marginal of $(U_f \times \mathrm{id})_*\Xi$, which is a fixed point of $\Fc_*$.
The second marginal of $(U_f \times \mathrm{id})_*\Xi$ is $\Xi$ itself.
This motivates us to call a measure on the product space $\Xc^{\Z_-} \times \Uc^{\Z_-}$ a {\bfi stochastic solution} of the state-space system for a given input measure $\Xi$ if it is a fixed point of $\Fc_*$ whose second marginal is $\Xi$.
Furthermore, we call a measure a {\bfi stochastic output} of the state-space system for a given input measure $\Xi$ if it is the first marginal of a stochastic solution.\footnote{
Some models consider an additional transformation $h \colon \Xc \rightarrow \Yc$ of the output.
In this case, the stochastic output is additionally transformed by the push-forward under the extension $\Xc^{\Z_-} \rightarrow \Yc^{\Z_-}$, $\Seq{x} \mapsto (h(\seq{x}{t}))_t$ of $h$ to sequence spaces.}
Three natural questions are whether a stochastic solution/output exists for a given stochastic input, whether it is unique, and whether, in some sense, it depends continuously on the input.
The first two questions are the natural stochastic generalization of the standard echo state property, and when they are answered affirmatively, we shall say that we are in the presence of the {\bfi stochastic echo state property}.
The third one is the generalization of the so-called {\bfi fading memory property} \cite{BoydChua1985}.
\cref{Banach_fp_dep_input} provides a general sufficient condition to ensure a positive answer to these three questions using a stochastic contractivity hypothesis that resembles that in \cite{ManjunathOrtega2023} and in which the continuity is defined using a Wasserstein metric.
The output processes whose existence we prove are causal in a specific sense and generalize those studied in purely deterministic situations.
In summary, \cref{Banach_fp_dep_input} spells out when a given state-space system can be used as a dynamic generative model for a prescribed set of dynamic covariates.

The paper is structured in five sections.
\cref{Preliminaries} introduces several important aspects of our framework, like the kind of processes that we shall be using as stochastic inputs, the topologies that are used in the different sequence spaces that appear in the paper, and some generalities about Wasserstein spaces.
\cref{The stochastic echo state property} contains the main result of the paper (\cref{Banach_fp_dep_input}) that, in the presence of certain stochastic contractivity and boundedness hypotheses, allows us to establish the stochastic echo state property, as well as the continuous dependence of the stochastic solutions generated by state-space systems on the inputs.
Various remarks and examples are included that illustrate the main hypotheses and results in the context of well-known models.
\cref{Proof of the stochastic echo state theorem} contains the proof of \cref{Banach_fp_dep_input} that proceeds by establishing first a more general result (\cref{Banach_fixed_point}) in which the uniqueness of fixed points problem is stripped from the context of state-space systems and is presented in an abstract, more general framework.
\cref{Conclusion} concludes the paper.
There is a glossary of symbols at the end of the paper to assist the reader in navigating the notation.

\section{Preliminaries}
\label{Preliminaries}

This section introduces several important aspects of our framework in detail.
\cref{Inputs generation} spells out the kind of processes that we shall be using as stochastic inputs.
\cref{Sequence spaces} explains in detail the topologies that are used in the different sequence spaces that appear in the extensions $F$ and $\Fc$ to sequence spaces of the state map $f$.
Finally, \cref{Wasserstein spaces} recalls some generalities about Wasserstein spaces and defines various probability spaces on product spaces that will be central in our developments.

\subsection{Inputs generation}
\label{Inputs generation}

The class of stochastic inputs we are interested in are those that arise as the stochastic output of a causal filter that is fed independent stochastic inputs.
More precisely, let $\Zc$ be a topological space and $V \colon \Zc^{\Z_-} \rightarrow \Uc^{\Z_-}$ be a continuous filter.
\nomenclature[78]{$V \colon \Zc^{\Z_-} \rightarrow \Uc^{\Z_-}$}{Input-generating filter}%
We will consider as input $\Seq{U} = V(\Seq{Z})$ for a $\Zc^{\Z_-}$-valued random variable $\Seq{Z}$.
We think of $\Seq{Z}$ as hidden underlying inputs arising in nature and of $V(\Seq{Z})$ as transformed or generated inputs that we can observe.
The hidden inputs $\Seq{Z}$ are assumed independent but the filter $V$ creates correlation between the marginals of $\Seq{U}$.
In particular, the past and the future of the observed inputs could be correlated.
However, in a physical system it is reasonable to assume that observed inputs are fully determined by the past and present hidden inputs.
To capture this mathematically, we assume the filter $V$ to be causal.
Causality refers exactly to the property that the output of $V$ at a time $t$ depends only on past and present inputs, that is, if $\seq{z}{s} = \seq{z}{s}'$ for all $s \leq t$, then $V(\Seq{z})_t = V(\Seq{z}')_t$.
Summing up, we consider stochastic inputs of the type $\Seq{U} = V(\Seq{Z})$ for a random variable $\Seq{Z}$ with independent marginals and a continuous causal filter $V$.
We shall refer to $V$ as the {\bfi input-generating filter}.
If the filter $V$ is time-invariant, these stochastic processes are known as causal {\bfi Bernoulli shifts} and comprise a wide class of processes \cite{AlquierWinten2012, DedeckerEtal2007, RC10}.

\subsection{Sequence spaces}
\label{Sequence spaces}

Let $\ZS = \Zc^{\Z_-}$
\nomenclature[76]{$\ZS = \Zc^{\Z_-}$}{Input space for input generation}%
and $\US = V(\ZS)$.
\nomenclature[79]{$\US = V(\ZS)$, $d_{\US}$}{Input space generated by causal filter $V$}%
The spaces $\Uc$, $\US$, and $\Zc$ are assumed to be Polish.
Then, $\ZS$ is also Polish since it is equipped with the product topology.
Let $d_{\US}$ be any metric on $\US$ inducing its topology.
There is no need to specify a metric on $\ZS$.
For the state space, assume $\Xc$ is also Polish and comes with a metric $d_{\Xc}$.
We want to equip the space of state sequences with a weighted $\ell^1$-distance.
For the weighting, fix a monotone sequence $\Seq{w} = (\seq{w}{t})_{t \leq -1} \subseteq (0,1)$ with $\sum_{t \leq -1} \seq{w}{t} = 1$ and
\begin{equation*}
	\abs{\Seq{w}}
	:= \sup_{n \in \N} \left( \sup_{t \leq -1} \frac{\seq{w}{t}}{\seq{w}{t-n}} \right)^{1/n}
	< \infty.
\end{equation*}
\nomenclature[80]{$\Seq{w}$, $\abs{\Seq{w}}$}{Weighting sequence with $\abs{\Seq{w}} = \sup_{n \in \N} ( \sup_{t \leq -1} \seq{w}{t}/\seq{w}{t-n} )^{1/n} < \infty$}%
For example, $\seq{w}{t} = (\gamma-1) \gamma^t$ with $\abs{\Seq{w}} = \gamma$ for $\gamma > 1$.
Throughout, $\N$ denotes the set of positive integers excluding zero.
\nomenclature[01]{$\N$}{Set of strictly positive integers}%
Fix $x_* \in \Xc$ and let $\Seq{x}^0 \in \Xc^{\Z_-}$ be the constant sequence $\seq{x}{t}^0 = x_*$.
Let $\XS = \ell^1_{\Seq{w}}(\Xc,\Seq{x}^0) \subseteq \Xc^{\Z_-}$ 
\nomenclature[82]{$\XS = \ell^1_{\Seq{w}}(\Xc,\Seq{x}^0) \subseteq \Xc^{\Z_-}$, $d_{\XS}$}{Sequences $\Seq{x} \in \Xc^{\Z_-}$ that satisfy $d_{\XS}(\Seq{x},\Seq{x}^0) = \sum_{t \leq -1} \seq{w}{t} d_{\Xc}(\seq{x}{t},\seq{x}{t}^0) < \infty$}%
be the subset of the semi-infinite sequence space containing all elements $\Seq{x} \in \Xc^{\Z_-}$ for which the series $\sum_{t \leq -1} \seq{w}{t} d_{\Xc}(\seq{x}{t},\seq{x}{t}^0)$ is finite.
The weighted $\ell^1$-distance $d_{\XS}(\Seq{x}^1,\Seq{x}^2) := \sum_{t \leq -1} \seq{w}{t} d_{\Xc}(\seq{x}{t}^1,\seq{x}{t}^2)$ defines a complete metric on $\XS$.
The topology induced by $d_{\XS}$ is separable and is at least as fine as the subspace topology induced by the product topology on $\Xc^{\Z_-}$.
In fact, the topology induced by $d_{\XS}$ is strictly finer than the subspace product topology if and only if $(\Xc,d_{\Xc})$ is unbounded (see \cref{app_lem_sequ_topology} in \cref{app_sec_pf_remarks}).
Weighted sequence spaces are of general interest in functional analysis;
for their role in the context of state-space systems, the reader may find extensive material in \cite{RC9}.
The function $F$, which we recall extends the state map $f$ to sequence spaces, is continuous with respect to the product topologies since $f$ is assumed to be continuous.
But without more knowledge on $f$ such as the condition in \cref{rem_on_cont_assumption} below it could happen that $F$ is not continuous with respect to the finer topology induced by the weighted $\ell^1$-distance.
We make this a standing assumption.

\begin{assumption}
\label{assumption_cont}
	The inclusion $F(\XS \times \US) \subseteq \XS$ holds, and the restriction $F \colon \XS \times \US \rightarrow \XS$ is continuous.
\end{assumption}

\begin{remark}
\label{rem_on_cont_assumption}
If there exists some $C \geq 0$ such that $d_{\Xc}(f(x,u),x_*) \leq C(1 + d_{\Xc}(x,x_*))$ for all $x \in \Xc$ and $u \in \Uc$, then \cref{assumption_cont} is satisfied.
This statement is proved in \cref{pf_rem_on_cont_assumption} in \cref{app_sec_pf_remarks}.
From here on, we regard $\Fc \colon \XS \times \US \rightarrow \XS \times \US$ as having domain and codomain $\XS \times \US$.
\end{remark}

\subsection{Wasserstein spaces}
\label{Wasserstein spaces}

Whenever we are given two metric spaces $(\Yc_1,d_{\Yc_1})$ and $(\Yc_2,d_{\Yc_2})$, we equip $\Yc_1 \times \Yc_2$ with the metric
\begin{equation}
\label{eq_product_metric}
	d_{\Yc_1 \times \Yc_2}((y_1,y_2),(y_1',y_2'))
	= d_{\Yc_1}(y_1,y_1') + d_{\Yc_2}(y_2,y_2'),
\end{equation}
which metrizes the product topology.
All topological spaces are endowed with their Borel $\sigma$-algebra.
We recall that if $\Yc_1$ and $\Yc_2$ are separable, then the Borel $\sigma$-algebra of $\Yc_1 \times \Yc_2$ is generated by the product of the Borel $\sigma$-algebras of $\Yc_1$ and $\Yc_2$.
Given any topological space $\Yc$, we denote the set of all probability measures on $\Yc$ by $P(\Yc)$.
\nomenclature[09]{$P(\Yc)$}{Set of all probability measures}%
To declutter the notation, given two subsets $P^i(\Yc) \subseteq P(\Yc)$, $i=1,2$, we write $P^1 \cap P^2(\Yc)$ to denote $P^1(\Yc) \cap P^2(\Yc)$.
If $(\Yc,d_{\Yc})$ is a metrized Polish space and $p \in [1,\infty)$, we define the {\bfi Wasserstein $p$-space} $P_p(\Yc) $ as 
\begin{equation*}
	P_p(\Yc)
	= \left\{ \mu \in P(\Yc) \colon \ExpP{Y \sim \mu}{d_{\Yc}(Y,y_0)^p} < \infty \right\},
\end{equation*}
\nomenclature[10]{$P_p(\Yc)$}{Wasserstein $p$-space}%
which is independent of the choice $y_0 \in \Yc$, and we equip it with the {\bfi Wasserstein $p$-distance}
\begin{equation*}
	W_p(\mu,\nu)
	= \left( \inf_{\gamma} \ExpP{(Y_1,Y_2) \sim \gamma}{d_{\Yc}(Y_1,Y_2)^p} \right)^{1/p},
\end{equation*}
\nomenclature[12]{$W_p(\mu,\nu)$}{Wasserstein $p$-distance}%
where the infimum is taken over all couplings of $\mu$ and $\nu$, which makes $P_p(\Yc)$ a metrized Polish space.
This notation suppresses the reliance of $P_p(\Yc)$ on the metric $d_{\Yc}$, but this dependence should be kept in mind, especially for the space of state sequences, whose metric depends on the choice of weighting sequence.
An introduction to Wasserstein spaces can be found in \cite[Chapter 6]{Villani2009}.
Important for us is the result that a sequence $(\mu_n)_n \subseteq P_p(\Yc)$ converges to $\mu \in P_p(\Yc)$ with respect to $W_p$ if and only if it converges weakly and $\ExpP{Y \sim \mu_n}{d_{\Yc}(y,y_0)^p} \rightarrow \ExpP{Y \sim \mu}{d_{\Yc}(y,y_0)^p}$.
In particular, if $(\Yc,d_{\Yc})$ is bounded, then the Wasserstein topology agrees with the weak topology;
and if it is unbounded, then the Wasserstein topology is strictly finer than the weak topology.
We denote by $\pi_{\Yc_i} \colon \Yc_1 \times \Yc_2 \rightarrow \Yc_i$, $i=1,2$, the projections.
For any $\Xi \in P(\Yc_2)$, let $P^{\Xi}(\Yc_1 \times \Yc_2) = \{ \mu \in P(\Yc_1 \times \Yc_2) \colon (\pi_{\Yc_2})_* \mu = \Xi \}$ 
\nomenclature[14]{$P^{\Xi}(\Yc_1 \times \Yc_2) $}{$P^{\Xi}(\Yc_1 \times \Yc_2) = \{ \mu \in P(\Yc_1 \times \Yc_2) \colon (\pi_{\Yc_2})_* \mu = \Xi \}$}%
and $P_p^{\Xi}(\Yc_1 \times \Yc_2) = P^{\Xi} \cap P_p(\Yc_1 \times \Yc_2)$.
\nomenclature[16]{$P_p^{\Xi}(\Yc_1 \times \Yc_2)$}{$P_p^{\Xi}(\Yc_1 \times \Yc_2) = P^{\Xi} \cap P_p(\Yc_1 \times \Yc_2)$}%
Note that $P^{\Xi}(\Yc_1 \times \Yc_2)$ is closed in the weak topology and $P_p^{\Xi}(\Yc_1 \times \Yc_2)$ is closed with respect to $W_p$.

\section{The stochastic echo state property}
\label{The stochastic echo state property}

This is the central section of the paper.
We start by presenting in \cref{Contractive and bounded input measures} the stochastic contractivity and boundedness hypotheses that later on in \cref{Banach_fp_dep_input} will allow us to establish the stochastic echo state property, as well as the continuous dependence of the stochastic outputs generated by state-space systems on the inputs.
Both \cref{Contractive and bounded input measures,The main result} contain various remarks and examples that illustrate the main hypotheses and results in the context of well-known models.
Let us recap the definition of the (stochastic) echo state property discussed in the introduction.

\begin{definition}
\label{def_ESP}
	The state-space system governed by $f$ has the {\bfi echo state property} on a set of inputs $\US_0 \subseteq \US$ if for any $\Seq{u} \in \US_0$ there exists a unique $\Seq{x} \in \XS$ such that $(\Seq{x},\Seq{u})$ is a fixed point of $\Fc$ (also called {\bfi solution for input} $\Seq{u}$).
The system has the {\bfi stochastic echo state property} on a set of stochastic inputs $\Mc \subseteq P(\US)$ relative to a set $\Pc \subseteq P(\XS \times \US)$ if for any $\Xi \in \Mc$ there exists a unique $\mu \in \Pc \cap P^{\Xi}(\XS \times \US)$ that is a fixed point of $\Fc_*$ (also called {\bfi (stochastic) solution for input} $\Xi$).
\end{definition}

The reason for having the stochastic echo state property defined `relative to a set $\Pc$' is that the space $P(\XS \times \US)$ contains many measures that do not admit a physical interpretation (for example, a measure with respect to which the states depend only on future inputs).
Hence, it makes sense to consider a subset $\Pc$.
An appropriate choice for $\Pc$ will be discussed in \cref{The main result}.

\subsection{Contractive and bounded input measures}
\label{Contractive and bounded input measures}

In the deterministic case, the most typical assumption for establishing the echo state property is the contractivity of $f$ in the first entry.
In the stochastic case, we do that with a weaker condition.
It is convenient to phrase it as a property of the input measure instead of as a property of $f$, but this is chiefly a matter of taste.
Dealing with measures comes at the expense of needing guarantees that various integrals will be finite, which gives rise to the notion of a bounded input measure.
We denote the projections on $\ZS$ by $p^t \colon \ZS \rightarrow \Zc$, 
\nomenclature[77]{$p^t \colon \ZS \rightarrow \Zc$}{Projection on the $t$ entry}%
$t \leq -1$.
Throughout, the parameter $p \in [1,\infty)$ appearing as a power in many integrals and as the order of Wasserstein spaces is fixed.
The notion of stochastic contractivity is inspired by \cite[Definition 4.1]{ManjunathOrtega2023} but is formulated with a conditional expectation.\footnote{
Taking the expectation in \eqref{def_indep_contr_input} yields $\ExpP{\Seq{U} \sim V_*\Theta}{ d_{\Xc} ( f(x_1,\seq{U}{t}) , f(x_2,\seq{U}{t}) )^p } \leq \kappa d_{\Xc} ( x_1,x_2 )^p$, which recovers \cite[Definition 4.1]{ManjunathOrtega2023} for the input measure $\Xi = V_*\Theta$ if $p=1$.}

\begin{definition}[{\bfseries Stochastic contractivity and boundedness}]
\label{def_indep_contr_bounded_input}
\hspace{0em}
\begin{enumerate}[\upshape (i)]\itemsep=0em
\item
Given $\kappa \in (0,1)$, we say that $\Theta \in P(\ZS)$ has $\kappa$-{\bfi contractive marginals} if $\Theta$-a.s.\ for all $x_1,x_2 \in \Xc$ and all $t \leq -1$\footnote{
Instead of taking the same $\kappa \in (0,1)$ for all $t \leq -1$, one could also take $\kappa_t \in (0,\infty)$ and require $\lim_{t \rightarrow -\infty} \prod_{s=t}^{-1} \kappa_t = 0$.}
\begin{equation}
\label{def_indep_contr_input}
	\ExpC{\Seq{Z} \sim \Theta}{d_{\Xc} ( f(x_1,V(\Seq{Z})_t) , f(x_2,V(\Seq{Z})_t) )^p}{(\seq{Z}{s})_{s \leq t-1}}
	\leq \kappa d_{\Xc} ( x_1,x_2 )^p.
\end{equation}
We say that $\Theta$ has contractive marginals if it has $\kappa$-contractive marginals for some $\kappa \in (0,1)$.

\item
Given $C \in (0,\infty)$, we say that $\Xi \in P(\US)$ is a $C$-{\bfi bounded} input measure if
\begin{equation}
\label{def_bounded_input}
	\sum_{t \leq -1} \seq{w}{t} \ExpP{\Seq{U} \sim \Xi}{ d_{\Xc} ( f(x_*,\seq{U}{t}) , x_* )^p }
	\leq C.
\end{equation}
We say that $\Xi$ is a bounded input measure if it is a $C$-bounded input measure for some $C \in (0,\infty)$.

\item
We denote by $\Mc_p^{C,\kappa}(\US) \subseteq P_p(\US)$ 
\nomenclature[94]{$\Mc_p^{C,\kappa}(\US)$}{Subset of $P_p(\US)$ made out of $C$-bounded input measures that are generated from measures with independent $\kappa$-contractive marginals}%
the set of all $C$-bounded input measures $\Xi \in P_p(\US)$ that are of the form $\Xi = V_* \Theta$ for a $\Theta \in P(\ZS)$ with independent $\kappa$-contractive marginals.
\end{enumerate}
\end{definition}

\begin{remark}
\label{rem_bounded_input}
	For any $\Xi \in \Mc_p^{C,\kappa}(\US)$, the integral in \eqref{def_bounded_input} is finite also for any other choice of $x_*$, though possibly with a different constant $C$.
\end{remark}

\cref{rem_bounded_input} is proved in \cref{pf_rem_bounded_input} in \cref{app_sec_pf_remarks}.
Clearly, if $f(\{x_*\} \times \Uc)$ is bounded, which is the case for $f$ as in \cref{rem_on_cont_assumption}, then any $\Xi \in P(\US)$ is a bounded input measure.
For a stationary $\Xi \in P(\US)$, that is, $\Xi$ has identical marginals denoted $\xi \in P(\Uc)$, the sum in \eqref{def_bounded_input} simplifies to $\ExpP{u \sim \xi}{ d_{\Xc} ( f(x_*,u) , x_* )^p }$.
Now, we survey some concrete examples of state-space systems and discuss measures with contractive marginals and bounded input measures.

\begin{example}[{\bfseries GARCH process}]
\label{ex_GARCH}
	Consider the first-order {\bfi generalized autoregressive conditional heteroscedastic (GARCH)} model \cite{Bollerslev1986JEcon, Engle1982Econ, FrancqZakoian2019} given by
\begin{equation*}
\begin{split}
	r_t
	&= \sigma_t \eta_t,
	\\
	\sigma_t^2
	&= \omega + \alpha r_{t-1}^2 + \beta \sigma_{t-1}^2,
\end{split}
\end{equation*}
for some fixed constants $\omega,\alpha,\beta \geq 0$ and iid inputs $\eta_t$.
We can encode this model as a state-space system:
if $\seq{U}{t} = \eta_{t-1}$ and $\seq{X}{t} = f(\seq{X}{t-1},\seq{U}{t})$ (in a stochastic sense) with $f(x,u) = \omega + (\alpha u^2 + \beta) x$, then $\sigma_t^2 = \seq{X}{t}$ and $r_t = \sqrt{\seq{X}{t}} \seq{U}{t+1}$.
As the state space we take $\Xc = [0,\infty)$ to enforce $\seq{X}{t} \geq 0$.
For the inputs, we take $\Zc = \Uc = \R$ and $V$ the identity.
Then, the law of $(\seq{U}{t})_{t \leq -1}$ has contractive marginals if $\Exp{(\alpha \eta_t^2 + \beta)^p} < 1$.
With $p=1$ and the common assumption that $\eta_t$ has mean 0 and variance 1, we recover the condition $\alpha+\beta<1$, which is a sharp condition for the existence and uniqueness of second-order stationary solutions of the first-order GARCH model.
Taking $x_* = 0$, it is clear that the law of $(\seq{U}{t})_{t \leq -1}$ is a bounded input measure.
\end{example}

In the next two examples, $\norm{\cdot}_{\mathrm{op}}$ denotes the operator norm induced by a vector norm $\norm{\cdot}$.

\begin{example}[{\bfseries Time-varying VAR model/state-affine system}]
\label{ex_VAR}
	Let $\Uc \subseteq \R^m$ and $\Xc \subseteq \R^n$ with the standard topologies, and let $f(x,u) = A(u)x + b(u)$ for some continuous matrix-valued function $A$ and vector-valued function $b$.
This describes a {\bfi vector autoregressive (VAR)} model of order 1 with varying coefficients \cite{RC10}.
In systems theory and the reservoir computing literature, this type of model is better known as a {\bfi state-affine system} \cite{Sontag1979, VanMienNorCyrot1984, RC6, GononOrtega2020IEEE}.
Given a vector norm $\norm{\cdot}$, a measure $\Theta \in P(\ZS)$ has $\kappa$-contractive marginals if a.s.
\begin{equation}
\label{ex_VAR_contr}
	\ExpC{\Seq{Z} \sim \Theta}{\norm{A(V(\Seq{Z})_t)}_{\mathrm{op}}^p}{(\seq{Z}{s})_{s \leq t-1}}
	\leq \kappa.
\end{equation}
We can take $x_* = 0$ so that a $\Xi \in P(\US)$ is a $C$-bounded input measure if
\begin{equation}
\label{ex_VAR_bounded}
	\sum_{t \leq -1} \seq{w}{t} \ExpP{\Seq{U} \sim \Xi}{\norm{ b(\seq{U}{t}) }^p}
	\leq C.
\end{equation}
\end{example}

\begin{example}[{\bfseries Echo state networks (ESNs)}]
\label{ex_ESN}
	Let $\Uc \subseteq \R^m$, $\Xc = [-1,1]^n$, $n \geq m$, and $\Zc \subseteq \R^l$ with the standard topologies induced by the Euclidean norm $\norm{\cdot}_2$, and consider an {\bfi echo state network (ESN)} \cite{JaegerHaas2004, Jaeger2010}, for which the map $f$ is of the form $f(x,u) = \tanh(Ax + Cu + b)$ with $A \in \R^{n \times n}$, $C \in \R^{n \times m}$, and $b \in \R^n$.
Here, $\tanh$ is applied componentwise.
The matrix $C$ is assumed to have full rank, which is a weak assumption as matrices of full rank are generic and $C$ is typically randomly selected.
Suppose the hidden inputs $\Seq{Z} \sim \Theta \in P(\ZS)$ are such that the variance of $\norm{V(\Seq{Z})_t}_2$ given $(\seq{Z}{s})_{s \leq t-1}$ is $\Theta$-a.s.\ lower bounded by some $\eta > 0$, uniformly in $t$.
Then, there exists some $\epsilon > 0$ such that if $\norm{DAD^{-1}}_{2,\mathrm{op}} < 1 + \epsilon$ for some non-singular diagonal $n \times n$-matrix $D$, then $\Theta$ has contractive marginals with respective to the norm $\norm{x}_D := \norm{Dx}_2$.
Verifying this is not straightforward and is done in detail in \cref{app_sec_ESN_stoch}.
\end{example}

\begin{example}[{\bfseries Stochastic difference equation}]
	Let $\Uc = \Xc = \R$ and $h > 0$.
Let $\Seq{Z} = (\seq{Z}{t})_{t \in \Z_-}$ be a stochastic process with iid marginals and $W = (W_s)_{s \in (-\infty,0]}$ be a real-valued stochastic process such that a.s.\ $W_{th} = \sum_{s \leq t} V(\seq{Z}{s})$, $t \in \Z_-$.
In particular, $W$ could be a Brownian motion\footnote{
With an infinite negative time-horizon, a Brownian motion $(B_s)_{s \in (-\infty,0]}$ satisfies a.s.\ $\lim_{s \rightarrow -\infty} B_s = 0$ instead of $B_0 = 0$.} if $V$ is the identity, but having a non-trivial filter $V$ leads to correlation between the increments of $W$ and, hence, a much richer class of admissible processes.
Given coefficient functions $\alpha,\beta \colon \R \rightarrow \R$, consider the stochastic difference equation
\begin{equation*}
	Y_s - Y_{s-h}
	= \alpha(Y_{s-h}) h + \beta(Y_{s-h}) (W_s - W_{s-h}), \qquad s \in (-\infty,-0].
\end{equation*}
Then, $\seq{X}{t} := Y_{ht}$, $t \in \Z_-$, satisfies $\seq{X}{t} = f(\seq{X}{t-1},\seq{U}{t})$ with the state map $f(x,u) = 1 + \alpha(x) h + \beta(x) u$ and the stochastic input $\seq{U}{t} = V(\Seq{Z})_t = W_{ht} - W_{h(t-1)}$.
If $\seq{U}{-1}$ has finite $p$-th moment, then the law of $\Seq{U}$ is a bounded input measure.
Suppose $\alpha$ and $\beta$ are $L$-Lipschitz continuous with $hL < 1$ and that the derivative of $\alpha$ is almost everywhere upper bounded by $-\eta$ for some $\eta \in [0,L)$.
If there exists some $\delta < \eta h / L$ such that a.s.\ $\ExpC{}{\abs{V(\Seq{Z})_t}}{(\seq{Z}{s})_{s \leq t-1}} \leq \delta$, then $\Seq{Z}$ has $\kappa$-contractive marginals with $\kappa = 1 - \eta h + L \delta$.
In the special case, in which $W$ is a Brownian motion, the condition $\ExpC{}{\abs{V(\Seq{Z})_t}}{(\seq{Z}{s})_{s \leq t-1}} \leq \delta$ is clearly satisfied for a sufficiently small discretization stepsize $h$.
\end{example}

\subsection{The main result}
\label{The main result}

We now turn to the problem of finding stochastic solutions of the state-space system $f$, which can be encoded as fixed points of the map $\Fc_*$.
As noted below \cref{def_ESP}, the set $P(\XS \times \US)$ of all probability measures is too large for this purpose.
There could be several fixed points of $\Fc_*$ in $P(\XS \times \US)$, but only some of them are sensible solutions of the state-space system, which leads us to consider a subset of $P(\XS \times \US)$ and pose the fixed point problem in it.
We illustrate this in the familiar context of the deterministic echo state property in which a solution filter is available.
Indeed, if the associated filter $U_f$ exists and $\Xi = V_*\Theta$, then the fixed point of $\Fc_*$ we are looking for is the law of $(\Seq{X} , V(\Seq{Z}))$, where $\Seq{Z} \sim \Theta$ and $\Seq{X} = U_f \circ V(\Seq{Z})$.
In particular, note that if the hidden inputs are independent, then the past $(\seq{X}{s},\seq{Z}{s})_{s \leq t}$ is independent of the hidden future inputs $(\seq{Z}{t+1},\dots,\seq{Z}{-1})$ for any $t \leq -2$ by causality of $U_f$ and $V$.
This motivates the following definition.

\begin{definition}
\label{def_causal_measure}
	We say that $\mu \in P(\XS \times \US)$ is a $V$-{\bfi causal} measure if there exist an $\XS$-valued random variable $\Seq{X} = (\seq{X}{t})_{t \leq -1}$ and a $\ZS$-valued random variable $\Seq{Z} = (\seq{Z}{t})_{t \leq -1}$ such that
(i) $\mu$ is the law of $(\Seq{X},V(\Seq{Z}))$;
(ii) $\Seq{Z}$ has independent marginals;
and (iii) $(\seq{X}{s},\seq{Z}{s})_{s \leq t}$ is independent of $(\seq{Z}{t+1},\dots,\seq{Z}{-1})$ for any $t \leq -2$.
The set of all $V$-causal measures is denoted $P^{V\text{-}\mathrm{causal}}(\XS \times \US)$.
\nomenclature[86]{$P^{V\text{-}\mathrm{causal}}(\XS \times \US)$}{Set of $V$-causal measures in $P(\XS \times \US)$}%
\end{definition}

We point out that $\ZS$-valued random variables are exactly stochastic processes on $\Zc$.
Although an $\XS$-valued random variable defines a stochastic process on $\Xc$, the converse does not hold in general since the topology on $\XS$ introduced in \cref{Sequence spaces} may be strictly finer than the subspace topology induced by the product topology and, in that case, measurability of $\seq{X}{t}$ for each $t \leq -1$ is not sufficient to conclude measurability of $\Seq{X}$.

The set of $V$-causal measures encodes what we consider sensible solutions to the state-space system.
To establish existence and uniqueness of fixed points of $\Fc_*$ in $P^{V\text{-}\mathrm{causal}}(\XS \times \US)$, we need a second integrability guarantee in addition to that of bounded input measures.
Let $P_p^{\Seq{w}}(\XS \times \US)$ be the set
\begin{equation*}
	P_p^{\Seq{w}}(\XS \times \US)
	= \left\{ \mu \in P_p(\XS \times \US) \colon \sum_{t \leq -1} \seq{w}{t} \ExpP{\Seq{X} \sim (\pi_{\XS})_*\mu}{d_{\Xc} ( \seq{X}{t} , x_* )^p} < \infty \right\}.
\end{equation*}
\nomenclature[90]{$P_p^{\Seq{w}}(\XS \times \US) $}{$P_p^{\Seq{w}}(\XS \times \US) = \{ \mu \in P_p(\XS \times \US) \colon \sum_{t \leq -1} \seq{w}{t} \ExpP{\Seq{X} \sim (\pi_{\XS})_*\mu}{d_{\Xc} ( \seq{X}{t} , x_* )^p} < \infty \}$}%
This technical condition is only relevant if the state space is unbounded.
If $(\Xc,d_{\Xc})$ is bounded, then $P^{\Xi} \cap P_p^{\Seq{w}}(\XS \times \US) = P^{\Xi}(\XS \times \US)$ for any $\Xi \in P_p(\US)$.
We are now ready to state our main result that we shall prove later on in \cref{Proof of the stochastic echo state theorem}.

\begin{theorem}[{\bfseries Stochastic echo state theorem}]
\label{Banach_fp_dep_input}
	Let $C \in (0,\infty)$ and $\kappa \in (0,2^{1-p} \abs{\Seq{w}}^{-1})$.
For any $\Xi \in \Mc_p^{C,\kappa}(\US)$, the map $\Fc_*$ has a unique fixed point $\mu^{\Xi}$ in $P^{\Xi} \cap P^{V\text{-}\mathrm{causal}} \cap P_p^{\Seq{w}}(\XS \times \US)$.
\nomenclature[93]{$\mu^{\Xi}$}{Stochastic solution for input $\Xi$}%
Furthermore, the map $\Mc_p^{C,\kappa}(\US) \rightarrow P_p(\XS \times \US)$ given by $\Xi \mapsto \mu^{\Xi}$ is continuous with respect to Wasserstein distances.
\end{theorem}

This theorem establishes a very general sufficient condition for the stochastic echo state property of the state-space system that goes beyond the familiar deterministic one.
The continuous dependence in \cref{Banach_fp_dep_input} of the unique solution on the input process constitutes a stochastic analog of the fading memory property.
We point out that the familiar condition $\kappa < 1$ arises for deterministic systems on a bounded state space;
the stronger condition $\kappa \abs{\Seq{w}} < 1$ appears even in the deterministic framework when the state space is unbounded \cite{RC9}.
We made a choice in terms of presentation in that we first fixed an arbitrary weighting sequence $\Seq{w}$ that determines the state sequence space $\XS$ and then introduced other definitions and results.
One could have also fixed a contraction rate $\kappa \in (0,2^{1-p})$ first and then taken a weighting sequence $\Seq{w}$ that satisfies $\abs{\Seq{w}} < 2^{1-p} \kappa^{-1}$.
That way \cref{Banach_fp_dep_input} holds for any $\kappa \in (0,2^{1-p})$ as long as \cref{assumption_cont} and the condition in \eqref{def_bounded_input} hold for the specific weighting sequence $\Seq{w}$ that now depends on $\kappa$.
We point out that \cref{rem_on_cont_assumption} provides sufficient conditions for \cref{assumption_cont} independent of the weighting sequence.
The proof of \cref{Banach_fp_dep_input} will be carried out in \cref{section_Banach,section_proof}.
We now discuss this stochastic echo state property and showcase some examples.
Proofs for the next two remarks can be found in \cref{app_sec_pf_remarks} (see \cref{app_lem_stationarity,app_lem_fp_consistent}).

\begin{remark}
\label{rem_stationary}
	Suppose $V$ is time-invariant.
If $\Xi$ in \cref{Banach_fp_dep_input} is stationary, then so is the fixed point $\mu^{\Xi}$.
\end{remark}

\begin{remark}
\label{rem_consistent_fp}
	The motivation to study fixed points of $\Fc_*$ was drawn from $(U_f \times \mathrm{id}_{\US})_*\Xi$ if the associated filter $U_f$ exists.
We would hope that \cref{Banach_fp_dep_input} recovers these fixed points.
Indeed, as long as the associated filter is measurable\footnote{
Sufficient conditions for this measurability will be discussed in \cref{app_sec_pf_remarks} (\cref{app_lem_measurability}).
}
as a map $U_f \colon \US \cap U_f^{-1}(\XS) \rightarrow \XS$, we have $\mu^{\Xi} = (U_f \times \mathrm{id}_{\US})_*\Xi$.
Thus, the stochastic result is consistent with the deterministic case.
\end{remark}

\begin{remark}
	The main result in \cite{ManjunathOrtega2023} looks similar to \cref{Banach_fp_dep_input}.
The main difference is that \cite[Theorem 4.15]{ManjunathOrtega2023} concerns a fixed point $S(\Xi)$ of the map $G_*(\cdot \times \Xi) \colon P_1(\XS) \rightarrow P_1(\XS)$, where $G(\Seq{x},\Seq{u}) = (f(\seq{x}{t},\seq{u}{t}))_{t \leq -1}$.
If the fixed point $\mu^{\Xi}$ of $\Fc_*$ in \cref{Banach_fp_dep_input} happens to be a product measure, then $S(\Xi) = \mu^{\Xi}$;
otherwise, there is no clear relation between $S(\Xi)$ and $\mu^{\Xi}$.
Furthermore, our results show that item (ii) in \cite[Theorem 4.15]{ManjunathOrtega2023} is false in general.
Since $S(\Xi)$ does not consider the dependence structure between inputs and states, the dynamic interpretation of $S(\Xi)$ is unclear.
\end{remark}

Note that the Wasserstein metric is a tool in the proof of \cref{Banach_fp_dep_input}.
The existence and uniqueness thesis in that theorem is not a metric statement.
The metric only enters in the continuity of the dependence of the fixed point on the input.
On the other hand, the fixed point in \cref{Banach_fp_dep_input} can be obtained as the limit of $\Fc_*^n(\delta_{\Seq{x}^0} \times \Xi)$.
The convergence is exponentially fast with respect to the Wasserstein distance, which will be evident from the proof.
Put differently, if $\Seq{U} \sim \Xi$ and if we set $\Seq{Y}^1 = \Fc(\Seq{x}^0,\Seq{U})$ and $\Seq{Y}^{n+1} = \Fc(\Seq{Y}^n)$, then $(\Seq{Y}^n)_n$ is a Markov chain with stationary distribution $\mu^{\Xi}$.
The speed of convergence of this Markov chain to its stationary distribution is exponential with respect to the Wasserstein distance.
This is useful for numerical simulations since $\Fc$ is known and samples from $\Seq{U}$ are the inputs that are given to us.

\begin{example}
	We revisit the GARCH model from \cref{ex_GARCH}.
In the literature, a unique stochastic solution is shown to exist by proving the almost sure convergence of the series
\begin{equation*}
	h(\Seq{U})_t
	= \omega + \omega \sum_{k=1}^{\infty} \prod_{j=0}^{k-1} (\alpha \seq{U}{t-j}^2 + \beta)
\end{equation*}
for fixed $t \leq -1$ and then proving that $\sigma_t^2 = h(\Seq{U})_t$ and $r_t = \sqrt{h(\Seq{U})_t} \seq{U}{t+1}$ is indeed the unique solution \cite{FrancqZakoian2019}.
Close examination of the proof of \cref{Banach_fp_dep_input} for this example reveals that the unique stochastic solution therein is found as the Wasserstein limit of the law of the partial sums of that series.
Notably, the partial sums are treated as stochastic processes and not for a fixed $t$.
More precisely, if $\Xi$ denotes the law of $(\seq{U}{t})_{t \leq -1}$ and $h^n(\Seq{U})_t$ denotes the partial sum of $h(\Seq{U})_t$, then $\Fc_*^n(\delta_{\Seq{x}^0} \times \Xi)$ is the law of $(h^{n-1}(\Seq{U}),\Seq{U})$, and the Wasserstein limit of $\Fc_*^n(\delta_{\Seq{x}^0} \times \Xi)$ is the fixed point of $\Fc_*$.
Furthermore, our result establishes continuous dependence of the law of $(\sigma_t^2)_{t \leq -1}$ on $\Xi$ with respect to Wasserstein distances.
\end{example}

\begin{example}
	Let us also revisit the state-affine systems from \cref{ex_VAR}.
If the input space $\Uc$ is compact, then the deterministic echo state property is equivalent to the existence of a vector norm $\norm{\cdot}$ and some $\kappa \in (0,1)$ such that $\norm{A(u)}_{\mathrm{op}} \leq \kappa$ for all $u \in \Uc$ \cite{QRC1}.\footnote{
Although \cite{QRC1} does not state that the matrix norm is an induced operator norm, this is indeed the case as can be seen from the proof therein and \cite[Theorem 2.19]{Hartfiel2001}.}
In contrast, the conditions \eqref{ex_VAR_contr} and \eqref{ex_VAR_bounded} (in which the vector norm is arbitrary) are sufficient for the stochastic echo state property by \cref{Banach_fp_dep_input}.
The condition \eqref{ex_VAR_contr} is clearly weaker than the pointwise bound $\norm{A(u)}_{\mathrm{op}} \leq \kappa$.
This exhibits an important class of examples where the stochastic echo state property holds while the deterministic echo state property provably does not hold.
We emphasize that state-affine systems naturally appear in quantum reservoir computing as this is the form that quantum channels determined by completely positive and trace-preserving maps take once a Gell-Mann basis has been fixed (see \cite{QRC1, QRC2} for details).
Even though it has not yet been investigated in depth, considering stochastic inputs in this setup is particularly natural.
\end{example}

\begin{example}
\label{ex_ESN_continued}
	Next, we revisit the echo state networks from \cref{ex_ESN}.
Since $f$ is bounded, any input measure is a bounded input measure, and the the condition exhibited in \cref{ex_ESN} is sufficient to obtain the stochastic echo state property from \cref{Banach_fp_dep_input}.
In contrast, in the deterministic case, the weakest known condition that ensures the echo state property and is applicable in reasonable generality is $\inf_{D \in \Dc} \norm{DAD^{-1}}_{2,\mathrm{op}} < 1$, where $\Dc$ is the set of all non-singular diagonal $n \times n$-matrices \cite{BuehnerYoung2006}.
Furthermore, one can show in dimension $n=2$ that there are matrices $A$ for which $\inf_{D \in \Dc} \norm{DAD^{-1}}_{2,\mathrm{op}}$ is arbitrarily close to 1 and for which the deterministic echo state property does not hold \cite{YildizJaegerKiebel2012}; see also \cref{app_sec_ESN_det}.
Thus, the condition in \cref{ex_ESN} is strictly weaker than what is required in the deterministic case.
\end{example}

Since $C$ and $\kappa$ are arbitrary in \cref{Banach_fp_dep_input}, we obtain a fixed point $\mu^{\Xi}$ for any $\Xi \in \bigcup_{C,\kappa} \Mc_p^{C,\kappa}(\US)$, where the union is taken over all $(C,\kappa) \in (0,\infty) \times (0, 2^{1-p}\abs{\Seq{w}}^{-1})$.
But it does not follow that $\Xi \mapsto \mu^{\Xi}$ is continuous on that whole set.
As we see in the next proposition, sufficient conditions for this are the compactness of the state and input spaces.
This result shall be proved later on at the end of \cref{section_proof}.

\begin{proposition}
\label{Banach_fp_compact_dep_input}
	Suppose $\Xc$ and $\Zc$ are compact.
For any $\Xi \in \bigcup_{C,\kappa} \Mc_p^{C,\kappa}(\US)$, where the union is taken over $(C,\kappa) \in (0,\infty) \times (0, \abs{\Seq{w}}^{-1})$, the map $\Fc_*$ has a unique fixed point $\mu^{\Xi}$ in $P^{\Xi} \cap P^{V\text{-}\mathrm{causal}}(\XS \times \US)$ that depends continuously on $\Xi$ in the weak topology.
\end{proposition}

\section{Proof of the stochastic echo state theorem}
\label{Proof of the stochastic echo state theorem}

This section is devoted to proving \cref{Banach_fp_dep_input}.
The strategy that we shall follow consists of establishing first a more general result (\cref{Banach_fixed_point}) in which we strip the uniqueness of fixed points problem from the context of state-space systems and present it in an abstract, more general framework.
We do that first in the following \cref{section_Banach} and then in \cref{section_proof} we shall proceed to prove \cref{Banach_fp_dep_input} by introducing several lemmas that will embed the assumptions of \cref{Banach_fp_dep_input} into the abstract framework of \cref{section_Banach}.

The general theme of the proof is to establish convergence of the sequence of measures $\Fc_*^n(\delta_{\Seq{x}^0} \times \Xi)$ for a given input measure $\Xi$ and to show that the limit point is the unique fixed point of $\Fc_*$.
We will find ourselves in a scenario that looks akin to an application of the Banach fixed point theorem.
The conditions of the Banach fixed point theorem will not be satisfied exactly, but we will be able to salvage the argument used in the proof of the Banach fixed point theorem.
To execute this proof strategy, we will need to translate the stochastic contractivity \eqref{def_indep_contr_input} to a contraction on Wasserstein space (\cref{lemma_contractive_joint});
and we will need to ensure that the limit point of $\Fc_*^n(\delta_{\Seq{x}^0} \times \Xi)$ will remain both inside the set $P_p^{\Seq{w}}(\XS \times \US)$ (\cref{lemma_fp_w_int}) and inside the set of $V$-causal measures (\cref{lemma_causal_closed}).
The latter will require us to go one level deeper and work with the hidden inputs before returning to the level of the observed inputs.

Before we begin the abstract framework, let us introduce a few ingredients that will be used several times throughout.
The first is the following inequality, which is immediate from Jensen's inequality and convexity of the power function;
\begin{equation}
\label{eq_triangle}
	(a+b)^p
	\leq 2^{p-1}(a^p+b^p),
	\qquad a,b \geq 0,\, p \geq 1.
\end{equation}
Given an $n \in \N$ and a sequence $(\alpha_t)_{t \leq -1} \subseteq [0,\infty)$, the second ingredient is the bound
\begin{equation}
\label{eq_w_index}
	\sum_{t \leq -1} \seq{w}{t} \alpha_{t-n}
	\leq \left( \sup_{t \leq -1} \frac{\seq{w}{t}}{\seq{w}{t-n}} \right) \sum_{t \leq -1} \seq{w}{t-n} \alpha_{t-n}
	\leq \abs{\Seq{w}}^n \sum_{t \leq -1} \seq{w}{t} \alpha_t.
\end{equation}
The third ingredient is the choice of metric \eqref{eq_product_metric} on the product of two metric spaces.
Lastly, we recall the formula for the change of a measure $\nu \in P(\Yc_1)$ under a measurable map $\phi \colon \Yc_1 \rightarrow \Yc_2$;
\begin{equation*}
	\ExpP{Y_2 \sim \phi_*\nu}{g(Y_2)}
	= \ExpP{Y_1 \sim \nu}{g \circ \phi(Y_1)}.
\end{equation*}
We will mention each use of \eqref{eq_product_metric}, \eqref{eq_triangle}, and \eqref{eq_w_index} but we will not explicitly mention each change of measure to avoid cluttering the proofs.
There is a glossary of symbols at the end of the paper, which the reader may find helpful.

\subsection{Relatively unique fixed points in an abstract framework}
\label{section_Banach}

In this section, we derive a general fixed point result on Wasserstein spaces (\cref{Banach_fixed_point} below) that will be the main tool for proving \cref{Banach_fp_dep_input} later on.
As we already mentioned, the formulation of this result goes beyond state-space systems and naturally takes place in a more general framework.
To avoid an overload of notation, let $(\Vc,d_{\Vc})$ and $(\Wc,d_{\Wc})$ be two metrized Polish spaces.
Later on, $\Vc$ plays the role of $\XS$ and $\Wc$ the role of $\US$.
Fix $v_0 \in \Vc$ and a continuous map $\Hc \colon \Vc \times \Wc \rightarrow \Vc \times \Wc$ satisfying $\pi_{\Wc} \circ \Hc = \pi_{\Wc}$.
The map $\Hc$ takes the role of $\Fc$.
Introduce the sets $P_p^{\mathrm{inv}}(\Hc)$ and $P_p^{\mathrm{con}}(\Hc)$ as the subset of $P_p(\Vc \times \Wc)$ forward-invariant under $\Hc_*^n$ for all $n \in \N$, that is,
\begin{equation*}
	P_p^{\mathrm{inv}}(\Hc)
	= \{ \mu \in P_p(\Vc \times \Wc) \colon \Hc^n_* \mu \in P_p(\Vc \times \Wc) \text{ for all } n \in \N \},
\end{equation*}
\nomenclature[97]{$P_p^{\mathrm{inv}}(\Hc)$}{Set of measures with integrability preserved by $\Hc_*$ (abstract)}%
and the set of measures for which $\Hc$ satisfies a certain contraction property, namely
\begin{equation*}
\begin{split}
	P_p^{\mathrm{con}}(\Hc)
	&= \bigcup_{\substack{C \in (0,\infty) \\ \kappa \in (0,1)}} P_p^{\mathrm{con}}(\Hc,C,\kappa),
	\\
	P_p^{\mathrm{con}}(\Hc,C,\kappa)
	&= \left\{ \mu \in P_p(\Vc \times \Wc) \colon \forall\, n \in \N \colon \ExpP{(v,w) \sim \mu}{d_{\Vc \times \Wc} ( \Hc^n(v,w) , \Hc^n(v_0,w) )^p} \leq \kappa^n C \right\}.
\end{split}
\end{equation*}
\nomenclature[98]{$P_p^{\mathrm{con}}(\Hc)$}{Set of measures with contraction property (abstract)}%
Given $C \in (0,\infty)$ and $\kappa \in (0,1)$, we denote by $\Mc_p(\Hc,C,\kappa)$ the set of all $\Theta \in P_p(\Wc)$ with $\delta_{v_0} \times \Theta \in P_p^{\mathrm{inv}}(\Hc)$ and $\Hc_*(\delta_{v_0} \times \Theta) \in P_p^{\mathrm{con}}(\Hc,C,\kappa)$.
\nomenclature[96]{$\Mc_p(\Hc,C,\kappa)$}{Set of admissible input measures (abstract)}%

\begin{theorem}
\label{Banach_fixed_point}
	Let $C \in (0,\infty)$ and $\kappa \in (0,1)$.
For any $\Theta \in \Mc_p(\Hc,C,\kappa)$, the map $\Hc_*$ has a fixed point in $P_p^{\Theta}(\Vc \times \Wc)$ that is unique relative to $P_p^{\Theta}(\Vc \times \Wc) \cap P_p^{\mathrm{con}}(\Hc)$.
More precisely, the sequence $(\Hc^n_*(\delta_{v_0} \times \Theta))_n$ converges with respect to the Wasserstein distance to a limit point $\mu_{\infty} \in P_p^{\Theta}(\Vc \times \Wc)$ with $\Hc_*\mu_{\infty} = \mu_{\infty}$, and for all $\mu \in P_p^{\Theta}(\Vc \times \Wc) \cap P_p^{\mathrm{con}}(\Hc)$ with $\Hc_*\mu = \mu$ we have $\mu = \mu_{\infty}$.\footnote{
In other words, if $\mu_{\infty	} \in P_p^{\Theta}(\Vc \times \Wc) \cap P_p^{\mathrm{con}}(\Hc)$, then $\mu_{\infty}$ is the unique fixed point of $\Hc_*$ in that set, and if $\mu_{\infty	} \notin P_p^{\Theta}(\Vc \times \Wc) \cap P_p^{\mathrm{con}}(\Hc)$, then $\Hc_*$ has no fixed point in that set.
}
Furthermore, the map $\Mc_p(\Hc,C,\kappa) \rightarrow P_p(\Vc \times\Wc)$ given by $\Theta \mapsto \lim_{n \rightarrow \infty} \Hc^n_*(\delta_{v_0} \times \Theta)$ is continuous with respect to Wasserstein distances.
\end{theorem}

\begin{proof}
	Fix $\Theta \in \Mc_p(\Hc,C,\kappa)$, and denote by $\iota \colon \Vc \times \Wc \rightarrow \Vc \times \Wc$ the map $\iota(v,w) = (v_0,w)$.
We claim that $\Hc^n_*\mu \in P_p(\Vc \times \Wc)$ for any $\mu \in P_p^{\Theta}(\Vc \times \Wc) \cap P_p^{\mathrm{con}}(\Hc)$ and any $n \in \N$.
To see this, consider a $\mu \in P_p^{\Theta}(\Vc \times \Wc) \cap P_p^{\mathrm{con}}(\Hc)$.
Then, $\mu \in P_p^{\mathrm{con}}(\Hc,C',\kappa')$ for some $C' > 0$ and $\kappa' \in (0,1)$.
Fix any $w_0 \in \Wc$.
Use first the triangle equality, then \eqref{eq_triangle}, and then the defining property of $P_p^{\mathrm{con}}(\Hc,C',\kappa')$ to estimate
\begin{equation}
\label{pf_Banach_fp_1}
\begin{split}
	&\ExpP{(v,w) \sim \Hc^n_*\mu}{d_{\Vc \times \Wc} ( (v,w),(v_0,w_0) )^p}
	\\
	&\leq 2^{p-1} \ExpP{(v,w) \sim \mu}{ d_{\Vc \times \Wc} ( \Hc^n(v,w),\Hc^n(v_0,w) )^p + d_{\Vc \times \Wc} ( \Hc^n(v_0,w),(v_0,w_0) )^p }
	\\
	&\leq 2^{p-1} (\kappa')^n C' + 2^{p-1} \ExpP{(v,w) \sim (\Hc^n \circ \iota)_*\mu}{d_{\Vc \times \Wc} ( (v,w),(v_0,w_0) )^p}.
\end{split}
\end{equation}
Note that $(\Hc^n \circ \iota)_*\mu = \Hc^n_*(\delta_{v_0} \times \Theta)$ since $\mu \in P_p^{\Theta}(\Vc \times \Wc)$.
We have $\Hc^n_*(\delta_{v_0} \times \Theta) \in P_p(\Vc \times \Wc)$ since $\delta_{v_0} \times \Theta \in P_p^{\mathrm{inv}}(\Hc)$ by definition of $\Mc_p(\Hc,C,\kappa)$.
Thus, the last expectation in \eqref{pf_Banach_fp_1} is also finite, which concludes the claim that $\Hc^n_*\mu \in P_p(\Vc \times \Wc)$.
The claim and the fact that $\delta_{v_0} \times \Theta \in P_p^{\mathrm{inv}}(\Hc)$ guarantee that the Wasserstein distance $W_p(\Hc^n_*\mu,\Hc^n_*(\delta_{v_0} \times \Theta))$ is well-defined for all $\mu \in P_p^{\Theta}(\Vc \times \Wc) \cap P_p^{\mathrm{con}}(\Hc)$.
For any such $\mu$,
\begin{equation}
\label{pf_Banach_fp_2}
\begin{split}
	W_p( \Hc_*^n \mu , \Hc^n_*(\delta_{v_0} \times \Theta) )^p
	&= W_p( \Hc_*^n \mu , (\Hc^n \circ \iota)_*\mu )^p
	\\
	&\leq \ExpP{(v_1,w_1,v_2,w_2) \sim ( \Hc^n \times ( \Hc^n \circ \iota ) )_*\mu}{d_{\Vc \times \Wc} ( (v_1,w_1) , (v_2,w_2) )^p}
	\\
	&= \ExpP{(v,w) \sim \mu}{d_{\Vc \times \Wc} ( \Hc^n(v,w) , \Hc^n(v_0,w) )^p}
	\leq (\kappa')^n C',
\end{split}
\end{equation}
where $C'$ and $\kappa'$ are such that $\mu \in P_p^{\mathrm{con}}(\Hc,C',\kappa')$.
The map $\Hc_*$ is, in general, not a contraction, and we can not apply the classical formulation of the Banach fixed point theorem.
However, we can adapt the strategy of the proof of the Banach fixed point theorem, which will yield a fixed point.
It is convenient to abbreviate $\nu_n := \Hc^n_*(\delta_{v_0} \times \Theta)$.
Applying \eqref{pf_Banach_fp_2} with $\mu = \nu_1$, which is an element of $P_p^{\Theta}(\Vc \times \Wc) \cap P_p^{\mathrm{con}}(\Hc,C,\kappa)$ by definition of $\Mc_p(\Hc,C,\kappa) \ni \Theta$, we find for any $n \in \N$ that $W_p( \nu_{n+1} , \nu_n ) \leq q^n Q$ with $q = \kappa^{1/p}$ and $Q = C^{1/p}$.
Subsequently, for any $m,n \in \N$ with $n < m$,
\begin{equation*}
	W_p( \nu_m , \nu_n )
	\leq \sum_{k=0}^{m-n-1} W_p( \nu_{n+k+1} , \nu_{n+k} )
	\leq \sum_{k=0}^{m-n-1} q^{n+k} Q
	\leq \frac{q^n Q}{1-q}.
\end{equation*}
Thus, $(\nu_n)_n$ is a Cauchy sequence with respect to $W_p$ and admits a limit point $\mu_{\infty} \in P_p(\Vc \times \Wc)$.
This limit point belongs to $P_p^{\Theta}(\Vc \times \Wc)$ since this is a closed subset with respect to $W_p$.
Since $\nu_n$ converges also weakly to $\mu_{\infty}$ and $\Hc$ is assumed continuous, it follows that
\begin{equation*}
	\Hc_* \mu_{\infty}
	= \Hc_* \left( \lim_{\substack{n \rightarrow \infty \\ \text{weak}}} \Hc_*^n (\delta_{v_0} \times \Theta) \right) 
	= \lim_{\substack{n \rightarrow \infty \\ \text{weak}}} \Hc_*^{n+1} (\delta_{v_0} \times \Theta)
	= \mu_{\infty}.
\end{equation*}
Hence, $\mu_{\infty} \in P_p^{\Theta}(\Vc \times \Wc)$ is a fixed point of $\Hc_*$.
If $\mu \in P_p^{\Theta}(\Vc \times \Wc) \cap P_p^{\mathrm{con}}(\Hc)$ is any fixed point of $\Hc_*$, then
\begin{equation*}
	W_p( \mu , \mu_{\infty} )^p
	= \lim_{n \rightarrow \infty} W_p( \mu , \nu_n )^p
	= \lim_{n \rightarrow \infty} W_p( \Hc_*^n \mu , \nu_n )^p
	\leq \lim_{n \rightarrow \infty} (\kappa')^n C'
	= 0,
\end{equation*}
by \eqref{pf_Banach_fp_2} and where $C' > 0$ and $\kappa' \in (0,1)$ are such that $\mu \in P_p^{\mathrm{con}}(\Hc,C',\kappa')$.
To prove continuous dependence of the fixed point on the input measure, suppose $(\Theta_k)_k \subseteq \Mc_p(\Hc,C,\kappa)$ converges to $\Theta \in \Mc_p(\Hc,C,\kappa)$ with respect to $W_p$.
Abbreviate $\mu_{k,n} = \Hc^n_*(\delta_{v_0} \times \Theta_k)$ and $\mu_{k,\infty} = \lim_{n \rightarrow \infty} \Hc^n_*(\delta_{v_0} \times \Theta_k)$.
We need to show that $\mu_{k,\infty}$ converges to $\mu_{\infty}$ with respect to $W_p$.
Given $\epsilon > 0$, take $n \in \N$ such that $q^n Q/(1-q) < \epsilon/3$.
Since $(\mu_{k,n})_k$ converges to $\nu_n$ with respect to $W_p$, there exists a $k_0 \in \N$ such that $W_p(\nu_n,\mu_{k,n}) < \epsilon/3$ for all $k \geq k_0$.
Thus, for all $k \geq k_0$,
\begin{equation*}
	W_p( \mu_{\infty} , \mu_{k,\infty} )
	\leq W_p( \mu_{\infty} , \nu_n ) + W_p( \nu_n , \mu_{k,n} ) + W_p( \mu_{k,n} , \mu_{k,\infty} )
	\leq \frac{q^n Q}{1-q} + \frac{\epsilon}{3} + \frac{q^n Q}{1-q}
	< \epsilon.
\end{equation*}
\end{proof}

\subsection{Proof of the main result}
\label{section_proof}

We start by proving several lemmas that serve as intermediate steps.
The goal of these lemmas is to embed the assumptions of \cref{Banach_fp_dep_input} into the abstract framework of \cref{Banach_fixed_point} on the one hand;
and to transform the relative uniqueness of the fixed point in \cref{Banach_fixed_point} into proper uniqueness in the context of \cref{Banach_fp_dep_input} on the other hand.
After the proof of \cref{Banach_fp_dep_input}, we also prove \cref{Banach_fp_compact_dep_input}.
For any $C \in (0,\infty)$, let $P_p^{\Seq{w},C}(\XS \times \US)$ be the set
\begin{equation*}
	P_p^{\Seq{w},C}(\XS \times \US)
	= \left\{ \mu \in P_p(\XS \times \US) \colon \sum_{t \leq -1} \seq{w}{t} \ExpP{\Seq{X} \sim (\pi_{\XS})_*\mu}{d_{\Xc} ( \seq{X}{t} , x_* )^p} \leq C \right\}
\end{equation*}
so that $P_p^{\Seq{w}}(\XS \times \US) = \bigcup_{C \in (0,\infty)} P_p^{\Seq{w},C}(\XS \times \US)$.
The first step is to translate the contractivity \eqref{def_indep_contr_input} of the hidden input marginals into the contractivity in $P_p^{\mathrm{con}}(\Fc)$.

\begin{lemma}
\label{lemma_contractive_joint}
	Suppose $\Xi \in \Mc_p^{C',\kappa}(\US)$ for some $C' \in (0,\infty)$ and $\kappa \in (0,\abs{\Seq{w}}^{-1})$.
Then, for any $C \in (0,\infty)$,
\begin{equation*}
	P^{\Xi} \cap P^{V\text{-}\mathrm{causal}} \cap P_p^{\Seq{w},C}(\XS \times \US)
	\subseteq P_p^{\mathrm{con}}(\Fc,C,\kappa \abs{\Seq{w}}).
\end{equation*}
\end{lemma}

\begin{proof}
	Let $\mu \in P^{\Xi} \cap P^{V\text{-}\mathrm{causal}} \cap P_p^{\Seq{w},C}(\XS \times \US)$.
By the equality $\pi_{\US} \circ \Fc = \pi_{\US}$, the choice of metric \eqref{eq_product_metric} on $\XS \times \US$, and Jensen's inequality,
\begin{equation*}
	\ExpP{(\Seq{X},\Seq{U}) \sim \mu}{ d_{\XS \times \US} ( \Fc^n(\Seq{X},\Seq{U}) , \Fc^n(\Seq{x}^0,\Seq{U}) )^p }
	\leq \sum_{t \leq -1} \seq{w}{t} \ExpP{(\Seq{X},\Seq{U}) \sim \mu}{ d_{\Xc} ( \pi_{\XS} \circ \Fc^n(\Seq{X},\Seq{U})_t , \pi_{\XS} \circ \Fc^n(\Seq{x}^0,\Seq{U})_t )^p }.
\end{equation*}
Abbreviate $\alpha_{n,t} = \ExpP{(\Seq{X},\Seq{U}) \sim \mu}{ d_{\Xc} ( \pi_{\XS} \circ \Fc^n(\Seq{X},\Seq{U})_t , \pi_{\XS} \circ \Fc^n(\Seq{x}^0,\Seq{U})_t )^p }$.
Note that $\sum_{t \leq -1} \seq{w}{t} \alpha_{0,t} \leq C$ since $\mu$ belongs to $P_p^{\Seq{w},C}(\XS \times \US)$.
We claim that $\alpha_{n,t}$ satisfies the recursion $\alpha_{n,t} \leq \kappa \alpha_{n-1,t-1}$.
In particular, $\alpha_{n,t} \leq \kappa^n \alpha_{0,t-n}$.
Once the claim is proved, it follows from \eqref{eq_w_index} that
\begin{equation*}
	\ExpP{(\Seq{X},\Seq{U}) \sim \mu}{ d_{\XS \times \US} ( \Fc^n(\Seq{X},\Seq{U}) , \Fc^n(\Seq{x}^0,\Seq{U}) )^p }
	\leq \sum_{t \leq -1} \seq{w}{t} \kappa^n \alpha_{0,t-n}
	\leq (\kappa \abs{\Seq{w}})^n C.
\end{equation*}
It remains to prove the claim.
On the one hand, $\Xi$ is the law of $V(\Seq{Z})$ for a $\ZS$-valued random variable $\Seq{Z}$ with independent $\kappa$-contractive marginals.
On the other hand, $\mu$ is the law of $(\Seq{X}',V(\Seq{Z}'))$ for some random variables $\Seq{X}'$ and $\Seq{Z}'$ as in \cref{def_causal_measure}.
A priori, $\Seq{Z} \neq \Seq{Z}'$.
But since the law of $V(\Seq{Z}')$ is also $(\pi_{\US})_*\mu = \Xi$ and since we are working on Polish spaces, there is another random variable $\Seq{X}$ on $\XS$ such that the joint law of $(\Seq{X},V(\Seq{Z}))$ is also $\mu$; see \cite[Theorem 8.17]{Kallenberg2021}.
It is convenient to abbreviate $\beta(\Seq{x},\Seq{z}) := \pi_{\XS} \circ \Fc^{n-1}(\Seq{x},V(\Seq{z}))_{t-1}$ so that $\pi_{\XS} \circ \Fc^n(\Seq{x},V(\Seq{z}))_t = f(\beta(\Seq{x},\Seq{z}),V(\Seq{z})_t)$.
In particular,
\begin{equation*}
	\alpha_{n,t}
	= \Exp{ d_{\Xc} ( f(\beta(\Seq{X},\Seq{Z}),V(\Seq{Z})_t) , f(\beta(\Seq{x}^0,\Seq{Z}),V(\Seq{Z})_t) )^p }.
\end{equation*}
We point out that $\beta(\Seq{x},\Seq{z})$ only depends on $\seq{x}{t-n}$ and $(\seq{z}{s})_{s \leq t-1}$ by causality of $V$.
By the tower property of conditional expectation and the independence in \cref{def_causal_measure}, we can therefore write
\begin{equation*}
\begin{split}
	\alpha_{n,t}
	&= \Exp{ \ExpC{}{ d_{\Xc} ( f(\beta(\Seq{X},\Seq{Z}),V(\Seq{Z})_t) , f(\beta(\Seq{x}^0,\Seq{Z}),V(\Seq{Z})_t) )^p }{(\seq{X}{s},\seq{Z}{s})_{s \leq t-1}} }
	\\
	&= \Exp{ \ExpC{}{ d_{\Xc} ( f(x_1,V(\Seq{Z})_t) , f(x_2,V(\Seq{Z})_t) )^p }{(\seq{Z}{s})_{s \leq t-1}}_{x_1 = \beta(\Seq{X},\Seq{Z}) \, x_2 = \beta(\Seq{x}^0,\Seq{Z})} }.
\end{split}
\end{equation*}
Using $\kappa$-contractivity \eqref{def_indep_contr_input} of the marginals of $\Seq{Z}$, we conclude that
\begin{equation*}
\begin{split}
	\alpha_{n,t}
	&\leq \kappa \Exp{ d_{\Xc} ( \beta(\Seq{X},\Seq{Z}) , \beta(\Seq{x}^0,\Seq{Z}) )^p }
	= \kappa \alpha_{n-1,t-1}.
\end{split}
\end{equation*}
This finishes the proof.
\end{proof}

The second step is to ensure that the limit of $\Fc^n_*(\delta_{\Seq{x}^0} \times \Xi)$, which we eventually prove to exist, remains in the set $P_p^{\Seq{w}}(\XS \times \US)$.

\begin{lemma}
\label{lemma_fp_w_int}
	Suppose $\Xi \in \Mc_p^{C,\kappa}(\US)$ for some $C \in (0,\infty)$ and $\kappa \in (0,2^{1-p} \abs{\Seq{w}}^{-1})$.
Then, $\Fc^n_*(\delta_{\Seq{x}^0} \times \Xi) \in P_p^{\Seq{w}}(\XS \times \US)$ for all $n \in \N$.
Furthermore, if the sequence $\Fc^n_*(\delta_{\Seq{x}^0} \times \Xi)$ converges weakly, then its limit point also belongs to $P_p^{\Seq{w}}(\XS \times \US)$.
\end{lemma}

\begin{proof}
	Let $\mu_n = \Fc^n_*(\delta_{\Seq{x}^0} \times \Xi)$, and abbreviate $\alpha_{n,t} = \ExpP{\Seq{X} \sim (\pi_{\XS})_*\mu_n}{d_{\Xc} ( \seq{X}{t} , x_* )^p}$.
The desired assertion $\mu_n \in P_p^{\Seq{w}}(\XS \times \US)$ amounts to showing that $\mu_n \in P_p(\XS \times \US)$ and $\sum_{t \leq -1} \seq{w}{t} \alpha_{n,t} < \infty$.
In fact, it suffices to show $\sum_{t \leq -1} \seq{w}{t} \alpha_{n,t} < \infty$ since this together with \eqref{eq_triangle}, the choice of the metric \eqref{eq_product_metric} on $\XS \times \US$, and the fact that $\Xi \in P_p(\US)$ imply that $\mu_n \in P_p(\XS \times \US)$.
Take a $\ZS$-valued random variable $\Seq{Z}$ with independent $\kappa$-contractive marginals and such that $V(\Seq{Z})$ has law $\Xi$.
Denote $\beta(\Seq{x},\Seq{z}) := \pi_{\XS} \circ \Fc^{n-1}(\Seq{x},V(\Seq{z}))_{t-1}$ as in the proof of \cref{lemma_contractive_joint} so that $\alpha_{n,t} = \Exp{d_{\Xc} ( f( \beta(\Seq{x}^0,\Seq{Z}) , V(\Seq{Z})_t ) , x_* )^p}$.
Abbreviate also $\alpha_t' = \ExpP{\Seq{Z} \sim \Theta}{d_{\Xc} ( f(x_*,V(\Seq{Z})_t) , x_* )^p}$ so that, by \eqref{eq_triangle},
\begin{equation*}
	\alpha_{n,t}
	\leq 2^{p-1} \Exp{d_{\Xc} ( f( \beta(\Seq{x}^0,\Seq{Z}) , V(\Seq{Z})_t ) , f(x_*,V(\Seq{Z})_t) )^p} + 2^{p-1} \alpha_t'.
\end{equation*}
By the tower property of conditional expectation and the $\kappa$-contractivity \eqref{def_indep_contr_input} of the marginals of $\Seq{Z}$,
\begin{equation*}
\begin{split}
	&\Exp{d_{\Xc} ( f( \beta(\Seq{x}^0,\Seq{Z}) , V(\Seq{Z})_t ) , f(x_*,V(\Seq{Z})_t) )^p}
	\\
	&= \Exp{ \ExpC{}{ d_{\Xc} ( f(x,V(\Seq{Z})_t) , f(x_*,V(\Seq{Z})_t) )^p }{(\seq{Z}{s})_{s \leq t-1}}_{x = \beta(\Seq{x}^0,\Seq{Z})} }
	\leq \kappa \alpha_{n-1,t-1}.
\end{split}
\end{equation*}
This establishes the recursion $\alpha_{n,t} \leq 2^{p-1} \kappa \alpha_{n-1,t-1} + 2^{p-1} \alpha_t'$.
Since $\alpha_{0,t} = 0$, iterating this recursion yields
\begin{equation*}
	\alpha_{n,t}
	\leq 2^{p-1} \sum_{j=0}^{n-1} (2^{p-1} \kappa)^j \alpha_{t-j}'.
\end{equation*}
That $\Xi$ is a $C$-bounded input measure \eqref{def_bounded_input} means exactly $\sum_{t \leq -1} \seq{w}{t} \alpha_t' \leq C$.
By \eqref{eq_w_index}, $\sum_{t \leq -1} \seq{w}{t} \alpha_{t-j}' \leq C \abs{\Seq{w}}^j$.
Finally,
\begin{equation}
\label{pf_geom_series}
	\sum_{t \leq -1} \seq{w}{t} \alpha_{n,t}
	\leq 2^{p-1} \sum_{j=0}^{n-1} (2^{p-1} \kappa)^j \sum_{t \leq -1} \seq{w}{t} \alpha_{t-j}'
	\leq 2^{p-1} C \sum_{j=0}^{n-1} (2^{p-1} \kappa \abs{\Seq{w}})^j
	< \infty.
\end{equation}
This concludes that $\mu_n \in P_p^{\Seq{w}}(\XS \times \US)$ for all $n \in \N$.
Now, suppose $\mu_n$ converges weakly to some limit $\mu_{\infty} \in P^{\Xi}(\XS \times \US)$.
Abbreviate $\phi_m(\Seq{x}) = \min\{ m , \sum_{t = -m}^{-1} \seq{w}{t} d_{\Xc} ( \seq{x}{t} , x_* )^p \}$.
These form a non-decreasing family of bounded continuous functions with pointwise limit $\sum_{t \leq -1} \seq{w}{t} d_{\Xc} ( \seq{x}{t} , x_* )^p \in \R \cup \{\infty\}$.
By monotone convergence,
\begin{equation*}
\begin{split}
	\sum_{t \leq -1} \seq{w}{t} \ExpP{\Seq{X} \sim (\pi_{\XS})_*\mu_{\infty}}{d_{\Xc} ( \seq{X}{t} , x_* )^p}
	&= \lim_{m \rightarrow \infty} \ExpP{\Seq{X} \sim (\pi_{\XS})_*\mu_{\infty}}{\phi_m(\Seq{X})}
	\\
	&= \lim_{m \rightarrow \infty} \lim_{n \rightarrow \infty} \ExpP{\Seq{X} \sim (\pi_{\XS})_*\mu_n}{\phi_m(\Seq{X})}
	\\
	&\leq \limsup_{n \rightarrow \infty} \sum_{t \leq -1} \seq{w}{t} \ExpP{\Seq{X} \sim (\pi_{\XS})_*\mu_n}{d_{\Xc} ( \seq{X}{t} , x_* )^p}
	\leq \frac{2^{p-1} C}{1 - 2^{p-1} \kappa \abs{\Seq{w}}},
\end{split}
\end{equation*}
where the last inequality used the assumption $\kappa < 2^{1-p} \abs{\Seq{w}}^{-1}$ that guarantees convergence of the geometric series in \eqref{pf_geom_series}.
Thus, $\mu_{\infty} \in P_p^{\Seq{w}}(\XS \times \US)$.
\end{proof}

The next lemma bridges the input measures in the abstract fixed point result with the input measures from \cref{def_indep_contr_bounded_input}.

\begin{lemma}
\label{lemma_contractive_input}
	Let $C \in (0,\infty)$ and $\kappa \in (0,\abs{\Seq{w}}^{-1})$.
Then, $\Mc_p^{C,\kappa}(\US) \subseteq \Mc_p(\Fc,C,\kappa \abs{\Seq{w}})$.
\end{lemma}

\begin{proof}
	Let $\Xi \in \Mc_p^{C,\kappa}(\US)$.
We proved in \cref{lemma_fp_w_int} that $\Fc^n_*(\delta_{\Seq{x}^0} \times \Xi) \in P_p^{\Seq{w}}(\XS \times \US) \subseteq P_p(\XS \times \US)$ for all $n \in \N$.
Thus, $\delta_{\Seq{x}^0} \times \Xi \in P_p^{\mathrm{inv}}(\Fc)$.
To see that $\Fc_*(\delta_{\Seq{x}^0} \times \Xi) \in P_p^{\mathrm{con}}(\Fc,C,\kappa \abs{\Seq{w}})$, write $\Xi$ as the law of $V(\Seq{Z})$ for a $\ZS$-valued random variable $\Seq{Z}$ with independent marginals.
The measure $\Fc_*(\delta_{\Seq{x}^0} \times \Xi)$, being the law of $(F(\Seq{x}^0,V(\Seq{Z})),V(\Seq{Z}))$, belongs to $P^{V\text{-}\mathrm{causal}}(\XS \times \US)$ by causality of $V$.
Furthermore, that $\Xi$ is a $C$-bounded input measure \eqref{def_bounded_input} means exactly
\begin{equation*}
	\sum_{t \leq -1} \seq{w}{t} \ExpP{\Seq{X} \sim (\pi_{\XS})_*\Fc_*(\delta_{\Seq{x}^0} \times \Xi)}{d_{\Xc} ( \seq{X}{t} , x_* )^p}
	= \sum_{t \leq -1} \seq{w}{t} \ExpP{\Seq{U} \sim \Xi}{d_{\Xc} ( f(x_*,\seq{U}{t}) , x_* )^p}
	\leq C,
\end{equation*}
which shows that $\Fc_*(\delta_{\Seq{x}^0} \times \Xi) \in P_p^{\Seq{w},C}(\XS \times \US)$.
Thus, $\Fc_*(\delta_{\Seq{x}^0} \times \Xi) \in P^{\Xi} \cap P^{V\text{-}\mathrm{causal}} \cap P_p^{\Seq{w},C}(\XS \times \US)$.
By \cref{lemma_contractive_joint}, $\Fc_*(\delta_{\Seq{x}^0} \times \Xi) \in P_p^{\mathrm{con}}(\Fc,C,\kappa \abs{\Seq{w}})$.
\end{proof}

The final building block concerns $P^{V\text{-}\mathrm{causal}}(\XS \times \US)$.
Our fixed point will be the limit of $\Fc^n_*(\delta_{\Seq{x}^0} \times \Xi)$.
Clearly, $\delta_{\Seq{x}^0} \times \Xi$ belongs to $P^{V\text{-}\mathrm{causal}}(\XS \times \US)$, and one can show that $P^{V\text{-}\mathrm{causal}}(\XS \times \US)$ is invariant under $\Fc_*$ (we will see this in the proof of \cref{Banach_fp_dep_input} further below).
However, $P^{V\text{-}\mathrm{causal}}(\XS \times \US)$ is, in general, not closed in the Wasserstein topology, and we cannot readily conclude that the limit of $\Fc^n_*(\delta_{\Seq{x}^0} \times \Xi)$ also belongs to it.
Instead, we go one level deeper to the space of hidden input measures.
Returning to the level of the observed inputs will be done in the final proof of \cref{Banach_fp_dep_input}.
Let $P^{\mathrm{causal}}(\XS \times \ZS)$ be the set of all measures $\mu \in P(\XS \times \ZS)$ that are the law of $(\Seq{X},\Seq{Z})$ for some $\ZS$-valued random variable $\Seq{Z}$ with independent marginals and some $\XS$-valued random variable $\Seq{X}$ for which $(\seq{X}{s},\seq{Z}{s})_{s \leq t}$ is independent of $(\seq{Z}{t+1},\dots,\seq{Z}{-1})$ for any $t \leq -2$.
\nomenclature[85]{$P^{\mathrm{causal}}(\XS \times \ZS)$}{Set of causal measures in $P(\XS \times \ZS)$}%
Then, $\mu \in P^{V\text{-}\mathrm{causal}}(\XS \times \US)$ if and only if $\mu$ is the law of $(\Seq{X},V(\Seq{Z}))$ for some $(\Seq{X},\Seq{Z})$ whose law belongs to $P^{\mathrm{causal}}(\XS \times \ZS)$.

\begin{lemma}
\label{lemma_causal_closed}
	The set $P^{\mathrm{causal}}(\XS \times \ZS)$ is closed in the weak topology.
\end{lemma}

\begin{proof}
	Suppose $(\mu_n)_n \subseteq P^{\mathrm{causal}}(\XS \times \ZS)$ converges weakly to some $\mu \in P(\XS \times \ZS)$.
Take random variables $\Seq{X}^n = (\Seq{X}^n_t)_{t \leq -1}$ on $\XS$ and $\Seq{Z}^n = (\Seq{Z}^n_t)_{t \leq -1}$ on $\ZS$ whose joint law is $\mu_n$.
Consider the projections $p_t^0,p_t^1,p_t^2$ defined by $p_t^0(\Seq{x},\Seq{z}) = ((\seq{x}{s})_{s \leq t},\Seq{z})$, $p_t^1(\Seq{x},\Seq{z}) = (\seq{x}{s},\seq{z}{s})_{s \leq t}$, and $p_t^2(\Seq{x},\Seq{z}) = (\seq{z}{t+1},\dots,\seq{z}{-1})$.
Then, independence of $(\seq{X}{s},\seq{Z}{s})_{s \leq t}$ and $(\Seq{Z}^n_{t+1},\dots,\Seq{Z}^n_{-1})$ amounts to the equality of measures $(p_t^0)_*\mu_n = (p_t^1)_*\mu_n \otimes (p_t^2)_*\mu_n$.
Since the projections are continuous, $(p_t^i)_*\mu_n$ converges weakly to $(p_t^i)_*\mu$ for each $i = 0,1,2$.
That $\Xc$ and $\Zc$ are Polish ensures that taking weak limits commutes with taking products; see \cite[Theorem 8.4.10]{Bogachev2007}.
Thus,
\begin{equation*}
	(p_t^1)_*\mu \otimes (p_t^2)_*\mu
	= \lim_{\substack{n \rightarrow \infty \\ \text{weak}}} (p_t^1)_*\mu_n \otimes (p_t^2)_*\mu_n
	= \lim_{\substack{n \rightarrow \infty \\ \text{weak}}} (p_t^0)_*\mu_n
	= (p_t^0)_*\mu.
\end{equation*}
The independence of the marginals of the second component of $\mu$ is shown analogously.
This concludes the causality for the limit point $\mu$.
\end{proof}

We piece everything together to prove \cref{Banach_fp_dep_input}.

\begin{proof}[\Pf{Banach_fp_dep_input}]
	Let $\Xi \in \Mc_p^{C,\kappa}(\US)$.
By \cref{lemma_contractive_input}, $\Mc_p^{C,\kappa}(\US) \subseteq \Mc_p(\Fc,C,\kappa \abs{\Seq{w}})$.
By \cref{Banach_fixed_point}, $\Fc_*$ has a fixed point $\mu^{\Xi}$ in $P_p^{\Xi}(\XS \times \US)$ that depends continuously on $\Xi$ with respect to Wasserstein distances and is unique relative to $P_p^{\Xi}(\XS \times \US) \cap P_p^{\mathrm{con}}(\Fc)$.
Hence, it is unique relative to $P^{\Xi} \cap P^{V\text{-}\mathrm{causal}} \cap P_p^{\Seq{w}}(\XS \times \US)$ since this is a subset of $P_p^{\Xi}(\XS \times \US) \cap P_p^{\mathrm{con}}(\Fc)$ by \cref{lemma_contractive_joint}.
The fixed point $\mu^{\Xi}$ in \cref{Banach_fixed_point} is obtained as the limit with respect to the Wasserstein distance of the convergent sequence $(\Fc^n_*(\delta_{\Seq{x}^0} \times \Xi))_n$.
Hence, $\mu^{\Xi} \in P_p^{\Seq{w}}(\XS \times \US)$ by \cref{lemma_fp_w_int}.
It remains to show that $\mu^{\Xi} \in P^{V\text{-}\mathrm{causal}}(\XS \times \US)$.

Take a $\Theta \in P(\ZS)$ with independent marginals and $V_* \Theta = \Xi$.
Let $\rho \colon \XS \times \ZS \rightarrow \XS \times \US$ be given by $\rho(\Seq{x},\Seq{z}) = (\Seq{x},V(\Seq{z}))$ and $\Gc \colon \XS \times \ZS \rightarrow \XS \times \ZS$ be given by $\Gc(\Seq{x},\Seq{z}) = (F(\Seq{x},V(\Seq{z})) , \Seq{z})$.
Then, $\Fc \circ \rho = \rho \circ \Gc$ and, hence, $\Fc^n_*(\delta_{\Seq{x}^0} \times \Xi) = \rho_* \Gc^n_*(\delta_{\Seq{x}^0} \times \Theta)$.
Clearly, $\delta_{\Seq{x}^0} \times \Theta \in P^{\mathrm{causal}}(\XS \times \ZS)$.
By causality of $V$, $P^{\mathrm{causal}}(\XS \times \ZS)$ is invariant under $\Gc_*$.
Thus, the sequence $\nu_n := \Gc^n_*(\delta_{\Seq{x}^0} \times \Theta)$ belongs to $P^{\mathrm{causal}}(\XS \times \ZS)$.
We claim that this sequence is weakly convergent with some limit $\nu \in P(\XS \times \ZS)$.
Since $P^{\mathrm{causal}}(\XS \times \ZS)$ is closed in the weak topology by \cref{lemma_causal_closed}, it follows from this claim that the fixed point $\mu^{\Xi} = \rho_*\nu$ belongs to $\rho_* P^{\mathrm{causal}}(\XS \times \ZS) = P^{V\text{-}\mathrm{causal}}(\XS \times \US)$.
To see that $\nu_n$ converges weakly, we show that it converges even with respect to a Wasserstein distance.
Since $\ZS$ is Polish, we can pick a bounded metric $d_{\ZS}$ inducing the topology on $\ZS$.
Fix any $\Seq{z}^0 \in \ZS$.
The choice of metric \eqref{eq_product_metric} on $\XS \times \ZS$ and \eqref{eq_triangle} imply
\begin{equation*}
	\ExpP{(\Seq{X},\Seq{Z}) \sim \nu_n}{d_{\XS \times \ZS}( (\Seq{X},\Seq{Z}) , (\Seq{x}^0,\Seq{z}^0) )^p}
	\leq 2^{p-1} \ExpP{\Seq{X} \sim (\pi_{\XS})_*\nu_n}{d_{\XS}( \Seq{X} , \Seq{x}^0 )^p} + 2^{p-1} \ExpP{\Seq{Z} \sim \Theta}{d_{\ZS}( \Seq{Z} , \Seq{z}^0 )^p}.
\end{equation*}
Note that $(\pi_{\XS})_*\nu_n = (\pi_{\XS})_*\Fc^n_*(\delta_{\Seq{x}^0} \times \Xi)$.
Jensen's inequality and the fact that $\Fc^n_*(\delta_{\Seq{x}^0} \times \Xi)$ belongs to $P_p^{\Seq{w}}(\XS \times \US)$ by \cref{lemma_fp_w_int} show that $\ExpP{\Seq{X} \sim (\pi_{\XS})_*\nu_n}{d_{\XS}( \Seq{X} , \Seq{x}^0 )^p}$ is finite.
Thus, $\nu_n \in P_p(\XS \times \ZS)$.
That $\Xi \in \Mc_p^{C,\kappa}(\US) \subseteq \Mc_p(\Fc,C,\kappa \abs{\Seq{w}})$ implies $\Fc_*(\delta_{\Seq{x}^0} \times \Xi) \in P_p^{\mathrm{con}}(\Fc,C,\kappa \abs{\Seq{w}})$ by definition of $\Mc_p(\Fc,C,\kappa \abs{\Seq{w}})$.
The equality
\begin{equation*}
	d_{\XS \times \ZS}( \Gc^n(\Seq{x},\Seq{z}) , \Gc^n(\Seq{x}^0,\Seq{z}) )
	= d_{\XS \times \US}( \Fc^n \circ \rho(\Seq{x},\Seq{z}) , \Fc^n \circ \rho(\Seq{x}^0,\Seq{z}) )
\end{equation*}
and the fact that $\rho_*\nu_1 = \Fc_*(\delta_{\Seq{x}^0} \times \Xi) \in P_p^{\mathrm{con}}(\Fc,C,\kappa \abs{\Seq{w}})$ yield
\begin{equation*}
\begin{split}
	W_p(\nu_{n+1},\nu_n)^p
	&\leq \ExpP{(\Seq{X},\Seq{Z}) \sim \nu_1}{d_{\XS \times \ZS}( \Gc^n(\Seq{X},\Seq{Z}) , \Gc^n(\Seq{x}^0,\Seq{Z}) )^p}
	\\
	&= \ExpP{(\Seq{X},\Seq{U}) \sim \rho_*\nu_1}{d_{\XS \times \US}( \Fc^n(\Seq{X},\Seq{U}) , \Fc^n(\Seq{x}^0,\Seq{U}) )^p}
	\leq (\kappa \abs{\Seq{w}})^n C.
\end{split}
\end{equation*}
As in the proof of \cref{Banach_fixed_point}, we find that $(\nu_n)_n$ is a Cauchy sequence with respect to $W_p$.
\end{proof}

Finally, we consider the compact case, which permits us to upgrade the continuity statement.

\begin{proof}[\Pf{Banach_fp_compact_dep_input}]
	Since $\Xc$ is compact, the factor $2^{1-p}$ in the bound on $\kappa$ in \cref{Banach_fp_dep_input} can be dropped since \cref{lemma_fp_w_int} is the only intermediate step requiring this bound and this lemma becomes redundant for a bounded state space.
Thus, we obtain a unique fixed point $\mu^{\Xi}$ in $P^{\Xi} \cap P^{V\text{-}\mathrm{causal}} \cap P_p^{\Seq{w}}(\XS \times \US)$ for any $\Xi \in \bigcup_{C,\kappa} \Mc_p^{C,\kappa}(\US) =: \Mc$, where the union is taken over all $(C,\kappa) \in (0,\infty) \times (0, \abs{\Seq{w}}^{-1})$.
Since $\Xc$ is compact, so is $\XS$ (a proof of this fact can be found in \cref{app_sec_pf_remarks}, \cref{app_lem_sequ_topology}).
Compactness of $\Zc$ carries over to $\ZS$ and $\US = V(\ZS)$.
In particular, $P_p^{\Seq{w}}(\XS \times \US) = P(\XS \times \US)$.

Now, suppose for contradiction that $\Xi \mapsto \mu^{\Xi}$ was not continuous on $\Mc$ with respect to the weak topology.
Then, there exist a $\Xi \in \Mc$, a weakly open set $W \subseteq P(\XS \times \US)$ containing $\mu^{\Xi}$, and a sequence $(\Xi_n)_n \subseteq \Mc$ converging weakly to $\Xi$ such that $\mu^{\Xi_n} \notin W$ for all $n \in \N$.
We showed in the proof of \cref{Banach_fp_dep_input} that there exist elements $\nu_n \in P^{\mathrm{causal}}(\XS \times \ZS)$ with $\mu^{\Xi_n} = \rho_*\nu_n$, where $\rho \colon \XS \times \ZS \rightarrow \XS \times \US$, $(\Seq{x},\Seq{z}) \mapsto (\Seq{x},V(\Seq{z}))$.
The space $P(\XS \times \ZS)$ is compact in the weak topology.
Thus, some subsequence of $(\nu_n)_n$ converges, and its limit $\nu$ belongs to $P^{\mathrm{causal}}(\XS \times \ZS)$ by \cref{lemma_causal_closed}.
Then, $\mu = \rho_* \nu$ is the limit of a subsequence of $(\mu^{\Xi_n})_n$.
Thus, $\mu \notin W$.
At the same time, $\mu \in \rho_*P^{\mathrm{causal}}(\XS \times \ZS) = P^{V\text{-}\mathrm{causal}}(\XS \times \US)$.
By definition, $\Fc_* \mu^{\Xi_n} = \mu^{\Xi_n}$ and $(\pi_{\US})_* \mu^{\Xi_n} = \Xi_n$ for all $n \in \N$.
By continuity, $\Fc_* \mu = \mu$ and $(\pi_{\US})_* \mu = \Xi$.
Thus, $\mu \in P^{\Xi} \cap P^{V\text{-}\mathrm{causal}}(\XS \times \US)$ is a fixed point of $\Fc_*$.
By uniqueness of the fixed point, $\mu = \mu^{\Xi} \in W$, which contradicts $\mu \notin W$.
\end{proof}

\section{Conclusion}
\label{Conclusion}

In this paper, we have developed a probabilistic framework to study the dependence structure induced by deterministic discrete-time state-space systems between input and output processes.
We have formulated general sufficient conditions under which solution processes exist and are unique once an input process has been fixed, which is the natural generalization of the deterministic echo state property.
When those conditions are satisfied, the given state-space system becomes a generative model for probabilistic dependences between two sequence spaces.
Moreover, as it had already been pointed out in \cite{ManjunathOrtega2023} (in the absence of probabilistic dependence between input and output), we showed that in the stochastic framework, the echo state property can hold under contractivity conditions that are strictly weaker than those in deterministic situations.
This means that state-space systems can induce a probabilistic dependence structure between input and output sequence spaces even without a functional relation between these two spaces.
Moreover, those conditions guarantee that the output depends continuously on the input when using the Wasserstein metric.
The output processes whose existence is proved are shown to be causal in a specific sense and to generalize those studied in purely deterministic situations. 

The framework developed in the paper and, more specifically, the characterization of the stochastic solutions opens the door to investigating many related problems.
For instance, some of the state-space families introduced in the paper, like echo state networks and state-affine systems, have been shown \cite{RC6, RC7, RC20, RC12} to have universality properties in the approximation of fading memory input/output systems with deterministic inputs.
A related question that could be asked is the universality of these paradigms as process-generating mechanisms.
More explicitly, given a causal and stationary process in a Wasserstein space, can it be approximated arbitrarily well by the stochastic output of any of the universal families cited above using iid inputs?
The answer to this question is the subject of ongoing work by the authors, which will be spelled out in a separate paper.

On a related topic, we recall that the input/output universality of various families of state-space systems has been studied in a stochastic setup in \cite{GononOrtega2020IEEE} using $L^p$-norms induced by the law of the input, but always assuming that the deterministic echo state property holds, and therefore a filter is available.
Moreover, $L^p $-integrability properties in the output are imposed as a hypothesis that, in practice, is difficult to verify.
In ongoing work, we investigate whether the different stochastic setup we adopted in this paper would allow us to formulate an alternative collection of related results with weaker conditions.

\acks{The authors acknowledge partial financial support from the School of Physical and Mathematical Sciences of the Nanyang Technological University. The second author is funded by an Eric and Wendy Schmidt AI in Science Postdoctoral Fellowship at the Nanyang Technological University. We thank Lyudmila Grigoryeva, G Manjunath, Peter Ti\v{n}o, and Patrick Cheridito for helpful discussions and remarks.}

\printnomenclature{14em}

\appendix

\section{Fixed points for hidden inputs}

In \cref{Banach_fp_dep_input}, we found a unique $\mu^{\Xi}$ that depends continuously on $\Xi$.
The stochastic input $\Xi$ is of the form $\Xi = V_*\Theta$.
One may ask if we can find a unique $\hat{\mu}^{\Theta}$ as a function of $\Theta$ instead of as a function of $\Xi$.
Naturally, $\hat{\mu}^{\Theta}$ should satisfy $(\pi_{\XS})_* \hat{\mu}^{\Theta} = (\pi_{\XS})_* \mu^{V_*\Theta}$ and $(\pi_{\ZS})_* \hat{\mu}^{\Theta} = \Theta$.
To respect these defining properties while maintaining the same style of definition as for $\mu^{\Xi}$, we look for $\hat{\mu}^{\Theta}$ by looking for fixed points of $\Gc_*$, where $\Gc \colon \XS \times \ZS \rightarrow \XS \times \ZS$ is given by $\Gc(\Seq{x},\Seq{z}) = (F(\Seq{x},V(\Seq{z})),\Seq{z})$.
The maps $\Fc$ and $\Gc$ are conjugate to each other in the sense that $\Fc \circ \rho = \rho \circ \Gc$, where $\rho(\Seq{x},\Seq{z}) = (\Seq{x},V(\Seq{z}))$.
In fact, fixed points $\hat{\mu}^{\Theta}$ of $\Gc_*$ as desired exist and the fixed points in \cref{Banach_fp_dep_input} can be recovered from them: if $\Xi = V_*\Theta$, then $\mu^{\Xi} = \rho_*\hat{\mu}^{\Theta}$, which turns out to be independent of the choice of $\Theta$ satisfying $V_*\Theta = \Xi$.
The fixed points $\hat{\mu}^{\Theta}$ are presented in \cref{Banach_fp_general_indep_input} further below.
In fact, these fixed points appeared implicitly in the proof of \cref{Banach_fp_dep_input}.
Analogous to $\Mc_p^{C,\kappa}(\US)$, let $\Mc_p^{C,\kappa}(\ZS) \subseteq P_p(\ZS)$ be the set of all $\Theta \in P_p(\ZS)$ with independent $\kappa$-contractive marginals for which $V_*\Theta$ is a $C$-bounded input measure.
We have the following variant of \cref{Banach_fp_dep_input}, in which the stochastic output is seen as a function of the hidden underlying inputs.
The proof of
\cref{Banach_fp_general_indep_input}
is analogous to the proof of \cref{Banach_fp_dep_input} with $\Gc$ in place of $\Fc$.

\begin{theorem}
\label{Banach_fp_general_indep_input}
	Let $C \in (0,\infty)$ and $\kappa \in (0,2^{1-p} \abs{\Seq{w}}^{-1})$.
For any $\Theta \in \Mc_p^{C,\kappa}(\ZS)$, the map $\Gc_*$ has a unique fixed point $\mu^{\Theta}$ in $P^{\Theta} \cap P^{\mathrm{causal}} \cap P_p^{\Seq{w}}(\XS \times \ZS)$.
Furthermore, the map $\Mc_p^{C,\kappa}(\ZS) \rightarrow P_p(\XS \times \ZS)$ given by $\Theta \mapsto \mu^{\Theta}$ is continuous with respect to Wasserstein distances.
\end{theorem}

\section{Proofs of remarks}
\label{app_sec_pf_remarks}

We made several remarks that deserve proof.
The first one was about the topology on $\XS$ and was formulated in \cref{Sequence spaces}.

\begin{lemma}
\label{app_lem_sequ_topology}
	The topology on $\XS$ induced by $d_{\XS}$ is strictly finer than the subspace topology induced by the product topology if and only if $(\Xc,d_{\Xc})$ is unbounded.
In particular, if $\Xc$ is compact, then so is $\XS$.
\end{lemma}

\begin{proof}
	Suppose $(\Xc,d_{\Xc})$ is unbounded.
If the topologies agreed, then there existed a basis element $B$ of the product topology such that $\Seq{x}^0 \in B \cap \XS \subseteq \{ \Seq{x} \in \XS \colon d_{\XS}(\Seq{x},\Seq{x}^0) < 1 \}$.
Being a basis element, there exists some $T \leq -1$ with $\left( \prod_{t < T} \Xc \right) \times \{x_*\}^{\times T} \subseteq B$.
Take $x \in \Xc$ with $\sum_{t < T} \seq{w}{t} d_{\Xc}(x,x_*) > 1$, and let $\Seq{x} = (\dots,x,x_*,\dots,x_*)$ be the sequence that is constantly $x_*$ for the last $T$ entries and constantly $x$ prior.
Then, $\Seq{x} \in B$ but $d_{\XS}(\Seq{x},\Seq{x}^0) \in (1,\infty)$, a contradiction.

Now, suppose $(\Xc,d_{\Xc})$ is bounded.
We show that the identity map is continuous as a map $\Xc^{\Z_-} \rightarrow \XS$.
The product topology on $\Xc^{\Z_-}$ is induced by the metric $d_{\XS}'(\Seq{x}^1,\Seq{x}^2) = \sup_{t \leq -1} \seq{w}{t} d_{\Xc}( \seq{x}{t}^1,\seq{x}{t}^2 )$.
Given $\epsilon > 0$, take $T \leq -1$ such that $\sum_{t < T} \seq{w}{t} < \epsilon \mathrm{diam}(\Xc)^{-1}$.
If $d_{\XS}'(\Seq{x}^1,\Seq{x}^2) < \epsilon \seq{w}{T}$, then $d_{\Xc}(\seq{x}{t}^1,\seq{x}{t}^2) < \epsilon$ for all $t \in \{T,\dots,-1\}$ and, hence, $d_{\XS}(\Seq{x}^1,\Seq{x}^2) < 2\epsilon$.
\end{proof}

In the next lemma, we verify \cref{rem_on_cont_assumption}.

\begin{lemma}
\label{pf_rem_on_cont_assumption}
If there exists some $C \geq 0$ such that $d_{\Xc}(f(x,u),x_*) \leq C(1 + d_{\Xc}(x,x_*))$ for all $x \in \Xc$ and $u \in \Uc$, then \cref{assumption_cont} is satisfied.
\end{lemma}

\begin{proof}
	That $F(\XS \times \US) \subseteq \XS$ follows directly from the growth property of $f$.
Since $\Uc$ is Polish, we can metrize it with some $d_{\Uc}$.
Then, the product topology on $\US$ is metrized by $d_{\US}'(\Seq{u}^1,\Seq{u}^2) = \sup_{t \leq -1} \seq{w}{t} \min\left\{ 1 , d_{\Uc}( \seq{u}{t}^1,\seq{u}{t}^2 ) \right\}$.
Let $\Seq{x}^1 \in \XS$, $\Seq{u}^1 \in \US$, and $\epsilon > 0$ be given.
Pick $T \leq -1$ such that $\sum_{t < T} \seq{w}{t} < \epsilon$ and $\sum_{t < T} \seq{w}{t} d_{\Xc}(\seq{x}{t}^1,x_*) < \epsilon$.
Take $\eta \in (0,1)$ so small that $d_{\Xc}(f(\seq{x}{t-1}^1,\seq{u}{t}^1),f(x,u)) < \epsilon$ for all $t \in \{T,\dots,-1\}$ and all $x \in \Xc$, $u \in \Uc$ with $d_{\Xc}(\seq{x}{t-1}^1,x) + d_{\Uc}(\seq{u}{t}^1,u) < 2 \eta$.
Let $\delta = \min\{\seq{w}{T-1}\eta , \epsilon\}$.
Then, $d_{\Xc}(\seq{x}{t-1}^1,\seq{x}{t-1}^2) + d_{\Uc}(\seq{u}{t}^1,\seq{u}{t}^2) < 2 \eta$ for all $t \in \{T,\dots,-1\}$ and all $\Seq{x}^2 \in \XS$, $\Seq{u}^2 \in \US$ with $d_{\XS}(\Seq{x}^1,\Seq{x}^2) + d_{\US}'(\Seq{u}^1,\Seq{u}^2) < \delta$.
Thus, for all such $\Seq{x}^2 \in \XS$, $\Seq{u}^2 \in \US$ we conclude
\begin{equation*}
\begin{split}
	d_{\XS}( F(\Seq{x}^1,\Seq{u}^1) , F(\Seq{x}^2,\Seq{u}^2) )
	&= \sum_{t=T}^{-1} \seq{w}{t} d_{\Xc}( f(\seq{x}{t-1}^1,\seq{u}{t}^1) , f(\seq{x}{t-1}^2,\seq{u}{t}^2) ) + \sum_{t < T} \seq{w}{t} d_{\Xc}( f(\seq{x}{t-1}^1,\seq{u}{t}^1) , f(\seq{x}{t-1}^2,\seq{u}{t}^2) )
	\\
	&\leq \sum_{t=T}^{-1} \seq{w}{t} \epsilon + \abs{\Seq{w}} \sum_{t < T-1} \seq{w}{t} \left( d_{\Xc}( f(\seq{x}{t}^1,\seq{u}{t+1}^1) , x_* ) + d_{\Xc}( x_* , f(\seq{x}{t}^2,\seq{u}{t+1}^2) ) \right)
	\\
	&\leq \epsilon + C \abs{\Seq{w}} \sum_{t < T-1} \seq{w}{t} ( 2 + 2 d_{\Xc}( \seq{x}{t}^1 , x_* ) + d_{\Xc}( \seq{x}{t}^1 , \seq{x}{t}^2 ) )
	\leq (1 + 5 C \abs{\Seq{w}}) \epsilon.
\end{split}
\end{equation*}
This proves the continuity of $F$.
\end{proof}

In the next lemma, we verify \cref{rem_bounded_input}.

\begin{lemma}
\label{pf_rem_bounded_input}
	Suppose $\Xi \in \Mc_p^{C,\kappa}(\US)$ for some $C \in (0,\infty)$ and $\kappa \in (0,1)$.
Then, for any $x \in \Xc$, it holds $\sum_{t \leq -1} \seq{w}{t} \ExpP{\Seq{U} \sim \Xi}{ d_{\Xc} ( f(x,\seq{U}{t}) , x )^p } < \infty$.
\end{lemma}

\begin{proof}
	Similarly to \eqref{eq_triangle} with three summands instead of two, we have
\begin{equation*}
\begin{split}
	\ExpP{\Seq{U} \sim \Xi}{ d_{\Xc} ( f(x,\seq{U}{t}) , x )^p }
	&\leq 3^{p-1} \ExpP{\Seq{U} \sim \Xi}{ d_{\Xc} ( f(x,\seq{U}{t}) , f(x_*,\seq{U}{t}) )^p }
	\\
	&\quad + 3^{p-1} \ExpP{\Seq{U} \sim \Xi}{ d_{\Xc} ( f(x_*,\seq{U}{t}) , x_* )^p }
	+ 3^{p-1} d_{\Xc} ( x_* , x )^p.
\end{split}
\end{equation*}
The $\Seq{w}$-weighted sum of the second term is bounded by $C$ since $\Xi$ is a $C$-bounded input measure.
By the $\kappa$-contractivity, there is some $\ZS$-valued random variable $\Seq{Z}$ such that
\begin{equation*}
\begin{split}
	\ExpP{\Seq{U} \sim \Xi}{ d_{\Xc} ( f(x,\seq{U}{t}) , f(x_*,\seq{U}{t}) )^p }
	&= \Exp{ d_{\Xc} ( f(x,V(\Seq{Z})_t) , f(x_*,V(\Seq{Z})_t) )^p }
	\\
	&= \Exp{ \ExpC{}{d_{\Xc} ( f(x,V(\Seq{Z})_t) , f(x_*,V(\Seq{Z})_t) )^p}{(\seq{Z}{s})_{s \leq t-1}} }
	\leq \kappa d_{\Xc} ( x , x_* )^p.
\end{split}
\end{equation*}
In particular, $\sum_{t \leq -1} \seq{w}{t} \ExpP{\Seq{U} \sim \Xi}{ d_{\Xc} ( f(x,\seq{U}{t}) , x )^p } \leq 3^{p-1} \kappa d_{\Xc} ( x , x_* )^p + 3^{p-1} C + 3^{p-1} d_{\Xc} ( x , x_* )^p$.
\end{proof}

In the next lemma, we verify \cref{rem_stationary}.

\begin{lemma}
\label{app_lem_stationarity}
	Suppose $V$ is time-invariant.
In \cref{Banach_fp_dep_input}, if $\Xi$ is stationary, then so is $\mu^{\Xi}$.
\end{lemma}

\begin{proof}
	Let $T \colon \ZS \rightarrow \ZS$ denote the time shift operator $T(\Seq{z})_t = \seq{z}{t-1}$.
We use the same letter to denote the time shift operators on $\Xc^{\Z_-}$ and $\Uc^{\Z_-}$.
Note that $T(\XS) = \XS$.
Since $V$ is assumed to commute with $T$, we also have $T(\US) = \US$.
We have to show that $(T \times T)_*\mu^{\Xi} = \mu^{\Xi}$.
The measure $(T \times T)_*\mu^{\Xi}$ is a fixed point of $\Fc_*$ since $\Fc$ commutes with $(T \times T)$.
That $\pi_{\US} \circ (T \times T) = T \circ \pi_{\US}$ and $T_*\Xi=\Xi$ implies that $(T \times T)_*\mu^{\Xi} \in P^{\Xi}(\XS \times \US)$.
We use commutativity of $V$ and $T$ once more to convince ourselves that $(T \times T)_*\mu^{\Xi}$ is a $V$-causal measure.
It is easy to see that it also belongs to $P_p^{\Seq{w}}(\XS \times \US)$.
Thus, $(T \times T)_*\mu^{\Xi} = \mu^{\Xi}$ by uniqueness of the fixed point of $\Fc_*$.
\end{proof}

In the next lemma, we verify \cref{rem_consistent_fp}.

\begin{lemma}
\label{app_lem_fp_consistent}
	Suppose the state-space system defined by $f$ has the echo state property and the associated filter is measurable as a map $U_f \colon \US \cap U_f^{-1}(\XS) \rightarrow \XS$.
Then, the unique fixed point of $\Fc_*$ in \cref{Banach_fp_dep_input} is $\mu^{\Xi} = (U_f \times \mathrm{id}_{\US})_*\Xi$.
\end{lemma}

\begin{proof}
	We claim that
\begin{equation*}
	\sum_{t \leq -1} \seq{w}{t} \ExpP{\Seq{U} \sim \Xi}{d_{\Xc}(U_f(\Seq{U})_t,x_*)^p}
	< \infty.
\end{equation*}
Assume, for now, the claim holds.
Then, $\ExpP{\Seq{U} \sim \Xi}{d_{\XS}(U_f(\Seq{u}),\Seq{x}^0)^p} < \infty$ by Jensen's inequality.
In particular, $d_{\XS}(U_f(\Seq{u}),\Seq{x}^0) < \infty$ for $\Xi$-almost every $\Seq{u} \in \US$.
Because of this and the measurability assumption, we can regard $\mu := (U_f \times \mathrm{id}_{\US})_*\Xi$ as a measure on $\XS \times \US$ as opposed to a measure on $\Xc^{\Z_-} \times \US$.
It is clear that $\mu \in P^{\Xi}(\XS \times \US)$.
Furthermore, the following sum is finite by the claim;
\begin{equation*}
	\sum_{t \leq -1} \seq{w}{t} \ExpP{\Seq{X} \sim (\pi_{\XS})_* \mu}{d_{\Xc}(\seq{X}{t},x_*)^p}
	= \sum_{t \leq -1} \seq{w}{t} \ExpP{\Seq{U} \sim \Xi}{d_{\Xc}(U_f(\Seq{U})_t,x_*)^p}
	< \infty.
\end{equation*}
This and $\Xi \in P_p(\US)$ assert that $\mu \in P_p^{\Seq{w}}(\XS \times \US)$.
Since $\Xi \in \Mc_p^{C,\kappa}(\US)$, there is some $\ZS$-valued random variable $\Seq{Z}$ with independent $\kappa$-contractive marginals and such that $V(\Seq{Z}) \sim \Xi$.
Then, $\mu$ is the law of $(U_f \circ V(\Seq{Z}),V(\Seq{Z}))$.
Causality of $U_f$ and $V$ guarantee that $\mu$ belongs to $P^{V\text{-}\mathrm{causal}}(\XS \times \US)$.
It is the defining equation of $U_f$ that $\Fc \circ (U_f \times \mathrm{id}_{\US}) = U_f \times \mathrm{id}_{\US}$.
Thus, $\mu$ is a fixed point of $\Fc_*$.
Hence, $\mu^{\Xi} = \mu$ by uniqueness of the fixed point $\mu^{\Xi}$.
It remains to prove the claim.\footnote{
The claim is trivial if the state space $(\Xc,d_{\Xc})$ is bounded.}
Abbreviate $A_t = \ExpP{\Seq{U} \sim \Xi}{d_{\Xc}(U_f(\Seq{U})_t,x_*)^p}$ and $B_t = \ExpP{\Seq{U} \sim \Xi}{d_{\Xc}(f(x_*,\seq{U}{t}),x_*)^p}$.
By independence and contractivity of the marginals of $\Seq{Z}$ and causality of $V$, that is, by using \eqref{def_indep_contr_input} similarly as in the proof of \cref{lemma_contractive_joint},
\begin{equation*}
\begin{split}
	\ExpP{\Seq{U} \sim \Xi}{d_{\Xc}(U_f(\Seq{U})_t,f(x_*,\seq{U}{t}))^p}
	&= \Exp{d_{\Xc}( f( U_f(V(\Seq{Z}))_{t-1} , V(\Seq{Z})_t) , f(x_* , V(\Seq{Z})_t) )^p}
	\\
	&= \Exp{ \ExpC{}{d_{\Xc}( f( x_1 , V(\Seq{Z})_t) , f(x_* , V(\Seq{Z})_t) )^p}{(\seq{Z}{s})_{s \leq t-1}}_{x = U_f(V(\Seq{Z}))_{t-1}} }
	\\
	&\leq \kappa \Exp{d_{\Xc}( U_f(V(\Seq{Z}))_{t-1} , x_* )^p}
	= \kappa A_{t-1}.
\end{split}
\end{equation*}
Combining this with \eqref{eq_triangle} yields the recursion $A_t \leq 2^{p-1} \kappa A_{t-1} + 2^{p-1} B_t$.
That $\Xi$ is a $C$-bounded input measure \eqref{def_bounded_input} means exactly that $\sum_{t \leq -1} \seq{w}{t} B_t \leq C$.
This, the recursion for $A_t$, and \eqref{eq_w_index} show that $S := \sum_{t \leq -1} \seq{w}{t} A_t$ satisfies $S \leq 2^{p-1} \kappa \abs{\Seq{w}} S + 2^{p-1} C$.
It is a hypothesis in \cref{Banach_fp_dep_input} that $\kappa < 2^{1-p} \abs{\Seq{w}}^{-1}$ and, therefore, we conclude that $S \leq 2^{p-1} C / (1 - 2^{p-1} \kappa \abs{\Seq{w}}) < \infty$, which proves the claim.
\end{proof}

Below, we provide some sufficient conditions for the measurability of $U_f \colon \US \cap U_f^{-1}(\XS) \rightarrow \XS$ required in \cref{app_lem_fp_consistent}.
Be reminded that \cref{assumption_cont} is still standing.
Results along similar lines as \cref{app_lem_measurability} have been proved in \cite{RC7, RC9}.

\begin{lemma}
\label{app_lem_measurability}
	Suppose one of the following holds.
\begin{enumerate}[\upshape (i)]\itemsep=0em
\item
$\Xc$ is compact and the state-space system defined by $f$ has the echo state property; or

\item
$f$ is a contraction in the first entry with a rate $c \in (0,\abs{\Seq{w}}^{-1})$, that is, $d_{\Xc}(f(x_1,u),f(x_2,u)) \leq c d_{\Xc}(x_1,x_2)$ for all $x_1,x_2 \in \Xc$ and $u \in \Uc$.
\end{enumerate}
Then, the associated filter $U_f$ exists as a map $U_f \colon \US \rightarrow \XS$ and is continuous.
\end{lemma}

\begin{proof}
	If $\Xc$ is compact, then $\XS = \Xc^{\Z_-}$ by \cref{app_lem_sequ_topology}, and continuity of $U_f \colon \Uc^{\Z_-} \rightarrow \Xc^{\Z_-}$ has been proved in \cite{Manjunath2020ProcA}.
Now, suppose $f$ is a contraction in the first entry with rate $c$.
Then, the map $\XS \rightarrow \XS$, $\Seq{x} \mapsto F(\Seq{x},\Seq{u})$ is a $(c\abs{\Seq{w}})$-contraction for any fixed $\Seq{u} \in \US$.
Thus, the Banach fixed point theorem asserts the echo state property on $\US$ with solutions in $\XS$.
We show that $U_f \colon \US \rightarrow \XS$ is continuous.
Abbreviate $a_t = d_{\Xc}(U_f(\Seq{u})_t,U_f(\Seq{u}')_t)$ and $b_t = d_{\Xc}(f(U_f(\Seq{u}')_{t-1},\seq{u}{t}),f(U_f(\Seq{u}')_{t-1},\seq{u}{t}'))$.
Then, $a_t \leq c a_{t-1} + b_t$ and
\begin{equation*}
	\sum_{t \leq -1} \seq{w}{t} a_t
	\leq c \abs{\Seq{w}} \sum_{t \leq -1} \seq{w}{t} a_t + \sum_{t \leq -1} \seq{w}{t} b_t.
\end{equation*}
Hence,
\begin{equation*}
	d_{\XS}(U_f(\Seq{u}),U_f(\Seq{u}'))
	= \sum_{t \leq -1} \seq{w}{t} a_t
	\leq \frac{1}{1-c\abs{\Seq{w}}} \sum_{t \leq -1} \seq{w}{t} b_t
	= \frac{1}{1-c\abs{\Seq{w}}} d_{\XS}(F(U_f(\Seq{u}'),\Seq{u}),F(U_f(\Seq{u}'),\Seq{u}')).
\end{equation*}
Continuity of $F$ from \cref{assumption_cont} implies
\begin{equation*}
	\limsup_{\Seq{u} \rightarrow \Seq{u}'} \, d_{\XS}(U_f(\Seq{u}),U_f(\Seq{u}'))
	\leq \limsup_{\Seq{u} \rightarrow \Seq{u}'} \, \frac{1}{1-c\abs{\Seq{w}}} d_{\XS}(F(U_f(\Seq{u}'),\Seq{u}),F(U_f(\Seq{u}'),\Seq{u}'))
	= 0,
\end{equation*}
which establishes continuity of $U_f \colon \US \rightarrow \XS$.
\end{proof}

\section{On echo state networks}

\subsection{Contractive marginals for echo state networks}
\label{app_sec_ESN_stoch}

Here, we provide a proof for the contractivity statement about echo state networks claimed in \cref{ex_ESN}.
More explicitly, we aim at establishing the contractivity condition in \eqref{def_indep_contr_input} for arbitrary $p \geq 1$ with the metric $d_{\Xc}(x_1,x_2) = \norm{x_1-x_2}_D = \norm{D(x_1-x_2)}_2$ under the assumption that $\gamma := \norm{DAD^{-1}}_{2,\mathrm{op}} < 1 + \epsilon$, where $\epsilon$ is to be determined and $D$ is some non-singular diagonal $n \times n$-matrix.
Note that $\tanh$ is 1-Lipschitz continuous with respect to $\norm{\cdot}_D$.
Thus, for fixed $u$, the state map $f(\cdot,u)$ is $\gamma$-Lipschitz continuous with respect to $\norm{\cdot}_D$.
But the Lipschitz constant $\gamma$ is not tight on large portions of the state space.
Let $r > 0$, and denote $L_r := \sqrt{((L_r')^2+2)/3}$, where $L'_r \in (0,1)$ is the Lipschitz constant of $\tanh$ with respect to $\norm{\cdot}_2$ outside the Euclidean ball of radius $r$ around the origin.
Then, given two points $y_1,y_2 \in \R^n$ with $\norm{y_1}_2 \geq 2r$, we have $\norm{\tanh(y_1) - \tanh(y_2)}_D \leq L_r \norm{y_1-y_2}_D$.
Hence, if we denote $U_r^x := \{ u \in \Uc \colon \norm{ Ax + Cu + b }_2 \geq 2r \}$, then for any $u,x_1,x_2$ with $u \in U_r^{x_1}$
\begin{equation*}
	\norm{f(x_1,u) - f(x_2,u)}_D
	\leq L_r \norm{Ax_1-Ax_2}_D
	\leq L_r \gamma \norm{x_1-x_2}_D.
\end{equation*}
Let $\zeta_t[(\seq{Z}{s})_{s \leq t-1}]$ denote the conditional law of $V(\Seq{Z})_t$ given $(\seq{Z}{s})_{s \leq t-1}$.
Then,
\begin{equation*}
\begin{split}
	&\ExpC{}{\norm{ f(x_1,V(\Seq{Z})_t) - f(x_2,V(\Seq{Z})_t) }_D^p}{(\seq{Z}{s})_{s \leq t-1}}
	\\
	&\leq L_r^p \gamma^p \norm{x_1-x_2}_D^p \zeta_t[(\seq{Z}{s})_{s \leq t-1}](U_r^{x_1}) + \gamma^p \norm{x_1-x_2}_D^p \zeta_t[(\seq{Z}{s})_{s \leq t-1}](\Uc \backslash U_r^{x_1})
	\\
	&= \gamma^p \norm{x_1-x_2}_D^p ( 1 - (1-L_r^p) \zeta_t[(\seq{Z}{s})_{s \leq t-1}](U_r^{x_1}) ).
\end{split}
\end{equation*}
For any $x \in \Xc$ there exists a $v_x \in \Uc$ such that if $\norm{Cu-Cv_x}_2 \geq 2r$, then $u \in U_r^x$.
Since $C$ is injective, there exists an $s(r) > 0$ such that if $\norm{u-v}_2 \geq s(r)$, then $\norm{Cu-Cv}_2 \geq 2r$.
Thus, $\Uc \backslash \B_{s(r)}(v_x) \subseteq U_r^x$, where $\B_s(v)$ denotes the Euclidean open ball of radius $s$ about $v$.
Denote $\delta_{r,t}[(\seq{Z}{s})_{s \leq t-1}] := \inf_{v \in \Uc} \zeta_t[(\seq{Z}{s})_{s \leq t-1}](\Uc \backslash \B_{s(r)}(v))$.
Then,
\begin{equation*}
	\ExpC{}{\norm{ f(x_1,V(\Seq{Z})_t) - f(x_2,V(\Seq{Z})_t) }_D^p}{(\seq{Z}{s})_{s \leq t-1}}
	\leq \gamma^p \norm{x_1-x_2}_D^p ( 1 - (1-L_r^p) \delta_{r,t}[(\seq{Z}{s})_{s \leq t-1}] ).
\end{equation*}
Finally, the assumption that the variance of $\norm{V(\Seq{Z})_t}_2$ given $(\seq{Z}{s})_{s \leq t-1}$ is $\Theta$-a.s.\ lower bounded by some $\eta > 0$, uniformly in $t$, implies that $\delta_{r,t}[(\seq{Z}{s})_{s \leq t-1}]$ is $\Theta$-a.s.\ lower bounded by some $\delta > 0$, uniformly in $t$, for a sufficiently small $r$.
Thus,
\begin{equation*}
	\ExpC{}{\norm{ f(x_1,V(\Seq{Z})_t) - f(x_2,V(\Seq{Z})_t) }_D^p}{(\seq{Z}{s})_{s \leq t-1}}
	\leq \gamma^p \norm{x_1-x_2}_D^p (1 - (1-L_r^p) \delta).
\end{equation*}
This shows that $\Theta$ has contractive marginals with respect to $\norm{\cdot}_D$ if $\gamma < 1+\epsilon$, where $\epsilon = (1 - (1-L_r^p) \delta)^{-1/p}-1$.
We point out that $\epsilon$ does not depend on $D$.

\subsection{Echo state networks without the deterministic echo state property}
\label{app_sec_ESN_det}

We claimed in \cref{ex_ESN_continued} that one can find $2 \times 2$-matrices $A$ for which $\inf_{D \in \Dc} \norm{DAD^{-1}}_{2,\mathrm{op}}$ is arbitrarily close to 1 and for which the deterministic echo state property does not hold.
It has been shown in \cite{YildizJaegerKiebel2012} that the following choice for $A$ obstructs the deterministic echo state property for any $c > 0$;
\begin{equation*}
	A
	=
	\begin{pmatrix}
		0 & c^{-1/2}
		\\
		-c^{3/2} & c+1
	\end{pmatrix}.
\end{equation*}
Let $D$ have diagonal entries $d_1,d_2 \neq 0$, and let $d := d_1/d_2$.
The operator norm $\norm{DAD^{-1}}_{2,\mathrm{op}}$ is the square root of the largest eigenvalue of $S := (DAD^{-1})^T DAD^{-1}$.
Its largest eigenvalue is given by
\begin{equation*}
	\lambda_{\mathrm{max}}(S)
	= \frac{\mathrm{tr}(S)}{2} + \frac{\sqrt{\mathrm{tr}(S)^2 - 4c^2}}{2},
\end{equation*}
where $\mathrm{tr}(S) = c^3/d^2 + d^2/c + (c+1)^2 > 0$ is the trace of $S$.
To find the infimum of $\norm{DAD^{-1}}_{2,\mathrm{op}}$ over all $D \in \Dc$, we minimize $\lambda_{\mathrm{max}}(S)$ as a function of $d$ for fixed $c > 0$.
We have
\begin{equation*}
	\frac{\partial \lambda_{\mathrm{max}}(S)}{\partial d}
	= \frac{1}{2} \frac{\partial \mathrm{tr}(S)}{\partial d} \left( 1 + \frac{\mathrm{tr}(S)}{\sqrt{\mathrm{tr}(S)^2 - 4c^2}} \right)
	= \frac{d^4 - c^4}{cd^3} \left( 1 + \frac{\mathrm{tr}(S)}{\sqrt{\mathrm{tr}(S)^2 - 4c^2}} \right),
\end{equation*}
which is equal to zero if and only if $d=c$ or $d=-c$.
Since $\lambda_{\mathrm{max}}(S)$ depends on $d$ only through $d^2$, either choice $d=c$ or $d=-c$ leads to the same value.
Thus,
\begin{equation*}
	\inf_{D \in \Dc} \norm{DAD^{-1}}_{2,\mathrm{op}}^2
	= \lambda_{\mathrm{max}}(S)|_{d=c}
	= \frac{2c + (c+1)^2}{2} + \frac{\sqrt{(2c + (c+1)^2)^2 - 4c^2}}{2}.
\end{equation*}
In particular, for any $\epsilon > 0$ we can take $c>0$ small enough so that $\inf_{D \in \Dc} \norm{DAD^{-1}}_{2,\mathrm{op}} < 1 + \epsilon$, which is what we claimed.

\section{Necessity of the integrability constraint}

In our main result, \cref{Banach_fp_dep_input}, we discovered unique fixed points of $\Fc_*$ in the set $P^{\Xi} \cap P^{V\text{-}\mathrm{causal}} \cap P_p^{\Seq{w}}(\XS \times \US)$.
The requirement that the fixed point be in $P_p^{\Seq{w}}(\XS \times \US)$ is necessary, which we illustrate with an example.
Suppose $\Xc$ is a vector space, $x_* = 0$, and $d_{\Xc}$ is induced by a norm on $\Xc$.
Consider the state map $f(x,u) = \alpha x$ with $\alpha \in (0,1/2)$.
Then, any hidden input measure $\Theta \in P(\ZS)$ has $\kappa := \alpha^p$-contractive marginals.
For the weighting on $\XS$, take $\seq{w}{t} = (\gamma-1) \gamma^t$ with $\gamma > 1$.
Suppose $p > 1$.
Then, it is possible to pick $\gamma$ so that $\gamma \alpha > 1$ and $\gamma \alpha^p < 2^{1-p}$.
The latter inequality ensures that the hypothesis $\kappa < 2^{1-p} \abs{\Seq{w}}^{-1}$ in \cref{Banach_fp_dep_input} is satisfied.
Given any $\Xi \in \Mc_p^{C,\kappa}(\US)$, the unique fixed point in \cref{Banach_fp_dep_input} is $\mu^{\Xi} = \delta_{\Seq{0}} \otimes \Xi$, where $\Seq{0}$ is the sequence that is constantly zero.
However, if $\Seq{x}^{\alpha}$ denotes the sequence $\seq{x}{t}^{\alpha} = \alpha^t$, then $\mu^{\alpha} := \delta_{\Seq{x}^{\alpha}} \otimes \Xi$ is also a fixed point of $\Fc_*$.
That $\gamma \alpha > 1$ ensures that $\Seq{x} \in \ell^1$ and, in particular, $\mu^{\alpha} \in P_p(\XS \times \US)$.
It is clear that $\mu^{\alpha} \in P^{\Xi} \cap P^{V\text{-}\mathrm{causal}}(\XS \times \US)$.
This does not contradict \cref{Banach_fp_dep_input} because $\mu^{\alpha}$ does not belong to $P_p^{\Seq{w}}(\XS \times \US)$.
Indeed,
\begin{equation*}
	\sum_{t \leq -1} \seq{w}{t} \ExpP{\Seq{X} \sim (\pi_{\XS})_*\mu^{\alpha}}{d_{\Xc} ( \seq{X}{t} , 0 )^p}
	= (\gamma-1) \sum_{t \leq -1} (\gamma \alpha^p)^t
	\geq (\gamma-1) \sum_{t \leq -1} 2^{(1-p)t}
	= \infty.
\end{equation*}
Note that this example works only on an unbounded state space for otherwise the sequence $\Seq{x}^{\alpha}$ would not be well-defined.

\bib{acm}{bibfile_FR}

\end{document}